\documentclass[twoside]{article}

\usepackage[accepted]{aistats2020}
\usepackage{hyperref}
\usepackage{our_style}
\usepackage[capitalize]{cleveref}
\usepackage{natbib}
\usepackage{url}            % simple URL typesetting
\usepackage{graphicx}
\usepackage{epstopdf}
\begin{document}



\twocolumn[

\aistatstitle{Approximate Cross-Validation in High Dimensions with Guarantees}

\aistatsauthor{ William T. Stephenson \And Tamara Broderick}

\aistatsaddress{ MIT CSAIL \And  MIT CSAIL}]

\begin{abstract}
Leave-one-out cross-validation (LOOCV) can be particularly accurate among cross-validation (CV) variants for machine learning assessment tasks -- e.g., assessing methods' error or variability. But it is expensive to re-fit a model $N$ times for a dataset of size $N$. Previous work has shown that approximations to LOOCV can be both fast and accurate -- when the unknown parameter is of small, fixed dimension. But these approximations incur a running time roughly cubic in dimension -- and we show that, besides computational issues, their accuracy dramatically deteriorates in high dimensions. 
Authors have suggested many potential and seemingly intuitive solutions, but these methods have not yet been systematically evaluated or compared. We find that all but one perform so poorly as to be unusable for approximating LOOCV. Crucially, though, we are able to show, both empirically and theoretically, that one approximation can perform well in high dimensions -- in cases where the high-dimensional parameter exhibits sparsity. Under interpretable assumptions, our theory demonstrates that the problem can be reduced to working within an empirically recovered (small) support. This procedure is straightforward to implement, and we prove that its running time and error depend on the (small) support size even when the full parameter dimension is large.

\end{abstract}

\section{Introduction} \label{introduction}

Assessing the performance of machine learning methods is an important task in medicine, genomics, and other applied fields. Experts in these areas are interested in understanding methods' error or variability and, for these purposes, often turn to cross validation (CV); see, e.g., \citet{saeb:2017:cvExample,powers:2019:cvExample,carrera:2009:cvExample,joshi:2009:cvExample,chandrasekaran:2011:cvExample,biswal:2001:jackknifeExample,roff:1994:jackknifeExample}. Even after decades of use \citep{stone:1974:earlyCV,geisser:1975:earlyCV}, CV remains relevant in modern high-dimensional and complex problems. In these cases, CV provides, for example, better out-of-sample error estimates than simple test error or training error \citep{stone:1974:earlyCV}. Moreover, among variants of CV, leave-one-out CV (LOOCV) offers to most closely capture performance on the dataset size of interest. For instance, LOOCV is particularly accurate for out-of-sample error estimation \citep[Sec.\ 5]{arlot:2010:cvOverview}.\footnote{In the case of linear regression, LOOCV provides the least biased and lowest variance estimate of out-of-sample error among other CV methods \citep{burman:1989:cvExperiments}.}

Modern datasets, though, pose computational challenges for CV. For instance, CV requires running a machine learning algorithm many times, especially in the case of LOOCV. This expense has led to recent proposals to \emph{approximate} LOOCV \citep{obuchi:2016:linearALOO,obuchi:2018:logisticALOO,beirami:2017:firstALOO, rad:2018:detailedALOO, wang:2018:primalDualALOO, giordano:2018:ourALOO,xu:2019:ALOOAMP}.
Theory and empirics demonstrate that these approximations are fast and accurate -- as long as the dimension $D$ of the unknown parameter in a problem is low. Unfortunately a number of issues arise in high dimensions, the exact case of modern interest. First, existing error bounds for LOOCV approximations either assume a fixed $D$ or suffer from poor error scaling when $D$ grows with $N$. One might wonder whether the theory could be improved, but our own experiments (see, e.g., \cref{fig-errorScalingExample}) confirm that LOOCV approximations can suffer considerable error degradation in high dimensions in practice. Second, even if the approximations were accurate in high dimensions, these approximations require solving a $D$-dimensional linear system, which incurs an $O(D^3)$ cost.

Previous authors have proposed a number of potential solutions for one or both of these problems, but these methods have not yet been carefully evaluated and compared. (\#1) \citet{koh:2017:influenceFunctions} use a randomized solver \citep{agarwal:2016:lissa} successfully for qualitative analyses similar to high-dimensional approximate CV, so it is natural to think the same technique might speed up approximate CV in high dimensions.
Another option is to consider that the unknown parameter may effectively exist in some subspace with much lower dimension that $D$. For instance, $\ell_1$ regularization offers an effective and popular means to recover a sparse parameter support.\footnote{Note that sparsity, induced by $\ell_1$ regularization, is typically paired with a focus on generalized linear models (GLMs) since these models simplify when many parameters are set to zero, are tractable to analyze with theory, and typically form the building blocks for even more complex models.}
Since existing approximate CV methods require twice differentiability of the regularizer, they cannot be applied directly with an $\ell_1$ penalty. (\#2) Thus, a second proposal -- due to \citet{rad:2018:detailedALOO,wang:2018:primalDualALOO} -- is to apply existing approximate CV methods to a smoothed version of the $\ell_1$ regularizer. (\#3) A third proposal -- made by, e.g., \citet{burman:1989:cvExperiments} -- is to ignore modern approximate CV methods, and speed up CV by uniform random subsampling of LOOCV folds.

We show that all three of these methods fail to address the issues of approximate CV in high dimensions. (\#4) A fourth proposal -- due to \citet{rad:2018:detailedALOO,wang:2018:primalDualALOO,obuchi:2016:linearALOO,obuchi:2018:logisticALOO,beirami:2017:firstALOO} -- is to again consider $\ell_1$ regularization for sparsity. But in this case, the plan is to fit the model once with the full dataset to find a sparse parameter subspace and then apply existing approximate CV methods to only this small subspace.

In what follows, we demonstrate with both empirics and theory that proposal \#4 is the only method that is fast and accurate for assessing out-of-sample error. We emphasize, moreover, its simplicity and ease of implementation.
%We start by providing background on model assessment in Sec*** and existing approximate CV methods for large $N$ in Sec***. 
%We demonstrate how the random solver (proposal \#1) fails to improve upon these methods for approximate CV in high dimensions in \cref{sec-highDProblems} and \cref{app-stochasticHessianNoGood}.} 
On the theory side, we show in \cref{sec-theory} that proposal \#4 will work if exact LOOCV rounds recover a shared support. Our major theoretical contribution is to prove that, under mild and interpretable conditions, the recovered support is in fact shared across rounds of LOOCV with very high probability (\cref{sec-linearRegression,sec-logisticRegression}). \citet{obuchi:2016:linearALOO} have considered a similar setup and shown that the effect of the change in support is asymptotically negligible for $\ell_1$-regularized linear regression; however, they do not show the support is actually shared. Additionally, \citet{beirami:2017:firstALOO,obuchi:2018:logisticALOO} make the same approximation in the context of other GLMs but without theoretical justification. We justify such approximations by proving that the support is shared with high probability in the practical \emph{finite-data} setting -- even for the very high-dimensional case $D = o(e^N)$ -- for both linear and logistic regression (\cref{linearRegressionTheorem,logisticRegressionTheorem}). Our support stability result may be of independent interest and allows us to show that, with high probability under finite data, the error and time cost of proposal \#4 will depend on the support size -- typically much smaller than the full dimension -- rather than $D$. Our experiments in \cref{sec-experiments} on real and simulated data confirm these theoretical results.

\begin{figure}
  \centering
  \includegraphics[scale=.5]{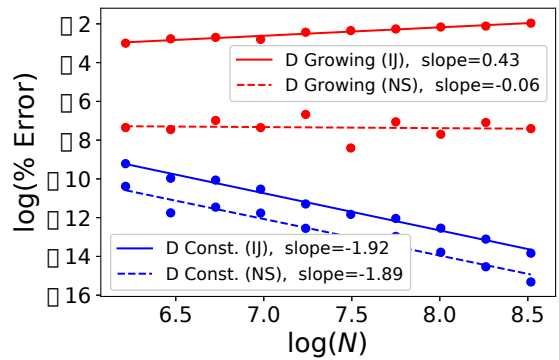}
  \caption{Log percent error (\cref{percentError}) of existing approximate LOOCV methods (``IJ'' and ``NS'') as a function of dataset size $N$ for $\ell_2$ regularized logistic regression. Dashed lines show \cref{modifiedApproximation} (``NS'') and solid show \cref{regularApproximation} (``IJ''). Blue lines have fixed data/parameter dimension $D$, while red lines have $D = N/10 $, although the true parameter has a fixed support size of $\deff = 2$ (see \cref{sec-approximation} for a full description). IJ and NS fail to capture this low ``effective dimension'' and suffer from substantially worse performance in high dimensions.}
   \label{fig-errorScalingExample}
\end{figure}

\textbf{Model assessment vs.\ selection.}
\citet{stone:1974:earlyCV,geisser:1975:earlyCV} distinguish at least two uses of CV: model assessment and model selection. Model assessment refers to estimating the performance of a single, fixed model. Model selection refers to choosing among a collection of competing models. We focus almost entirely on model assessment -- for two principal reasons. First, as discussed above, CV is widely used for model assessment in critical applied areas -- such as medicine and genetics. Before we can safely apply approximate CV for model assessment in these areas, we need to empirically and theoretically verify our methods. Second, historically, rigorous analysis of the properties of model selection even for \emph{exact} CV has required significant additional work beyond analyzing CV for model assessment. In fact, exact CV for model selection has only recently begun to be theoretically understood for $\ell_1$ regularized linear regression \citep{homrighausen:2013:l1Kfold,homrighausen:2014:l1LOO,chetverikov:2019:l1Kfold}. Our experiments in \cref{app-lambdaSelection} confirm that approximate CV for model selection exhibits complex behavior. We thus expect significant further work, outside the scope of the present paper, to be necessary to develop a theoretical understanding of approximate CV for model selection. Indeed, to the best of our knowledge, all existing theory for the accuracy of approximate CV applies only to model assessment \citep{beirami:2017:firstALOO,rad:2018:detailedALOO,giordano:2018:ourALOO,xu:2019:ALOOAMP,koh:2019:IJNSBounds}.

\section{Overview of Approximations} \label{sec-approximation}

Let $\theta \in \Theta \subseteq \R^{D}$ be an unknown parameter of interest.
Consider a dataset of size $N$, where $n \in [N] := \{1,2,\ldots,N\}$ indexes the data point.
Then a number of problems -- such as maximum likelihood, general M-estimation, and regularized loss minimization --
can be expressed as solving
\begin{equation}
	\hat\theta := \argmin_{\theta \in \Theta} \frac{1}{N} \sum_{n=1}^N f_n(\theta) + \lambda R(\theta),\label{ourOptProblem} \end{equation}
where $\lambda \geq 0$ is a constant, and $R: \Theta \to \R_+$ and $f_n: \Theta \to \R$ are functions. For instance, $f_n$ might be the loss associated with the $n$th data point, $R$ the regularizer, and $\lambda$ the amount of regularization. Consider a dataset where the $n$th data point has covariates $x_n \in \R^D$ and response $y_n \in \R$. In what follows, we will be interested in taking advantage of sparsity. With this in mind, we focus on generalized linear models (GLMs), where $f_n(\theta) = f(x_n^T \theta, y_n)$, as they offer a natural framework where sparsity can be expressed by choosing many parameter dimensions to be zero.

In LOOCV, we are interested in solutions of the same problem with the $n$th data point removed.\footnote{See \cref{app-CVReview} for a brief review of CV methods.} To that end,\footnote{Note our choice of $1/N$ scaling here -- instead of $1/(N-1)$. While we believe this choice is not of particular importance in the case of LOOCV, this issue does not seem to be settled in the literature; see \cref{app-LOOScaling}.} define $\thetan := \argmin_{\theta \in \Theta} \frac{1}{N} \sum_{m:\, m\neq n} f_m(\theta) + \lambda R(\theta)$. Computing $\thetan$ exactly across $n$ usually requires $N$ runs of an optimization procedure -- a prohibitive cost. Various approximations, detailed next, address this cost by solving \cref{ourOptProblem} only once.

\textbf{Two approximations.} Assume that $f$ and $R$ are twice differentiable functions of $\theta$. Let $F(\theta) := (1/N)\sum_n f(x_n^T \theta, y_n)$ be the unregularized objective, and let $H(\theta) := \nabla^2_{\theta} F(\theta) + \lambda\nabla^2_{\theta} R(\theta)$ be the Hessian matrix of the full objective. For the moment, we assume appropriate terms in each approximation below are invertible. \citet{beirami:2017:firstALOO,rad:2018:detailedALOO, wang:2018:primalDualALOO, koh:2019:IJNSBounds} approximate $\thetan$ by taking a Newton step (``NS'') on the objective $(1/N)\sum_{m:\, m\neq n} f_m + \lambda R$ starting from $\hat\theta$; see \cref{app-modifiedDerivation} for details. We thus call this approximation $\initNS$ for regularizer $R$:
\begin{equation}
	\initNS := \hat\theta + \frac{1}{N}\left(H(\hat\theta) - \frac{1}{N} \nabla^2_\theta f_n(\hat\theta)\right) \inv \nabla_\theta f_n(\hat\theta).
	\label{modifiedApproximation}
\end{equation}
In the case of GLMs, Theorem 8 of \cite{rad:2018:detailedALOO} gives conditions on $x_n$ and $f(\cdot, \cdot)$ that imply, for fixed $D$, the error of $\initNS$ averaged over $n$ is $o(1/N)$ as $N\to\infty$.

\citet{koh:2017:influenceFunctions,beirami:2017:firstALOO,giordano:2018:ourALOO,koh:2019:IJNSBounds} consider a second approximation. As their approximation is inspired by the \emph{infinitesimal jackknife} (``IJ'') \citep{jaeckel:1972:infinitesimal,efron:1982:jackknife}, we denote it by $\initIJ$; see \cref{app-regularDerivation}.
\begin{equation}
	 \initIJ := \hat\theta + \frac{1}{N} H(\hat\theta)\inv \nabla_\theta f_n(\hat\theta). 	\label{regularApproximation}
\end{equation}
\citet{giordano:2018:ourALOO} study the case of $\lambda = 0$, and, in their Corollary 1, show that the accuracy of \cref{regularApproximation} is bounded by $C/N$ in general or, in the case of bounded gradients $\|\nabla_\theta f(x_n^T\theta, y_n)\|_\infty \leq B$, by $C'B/N^2$. The constants $C,C'$ may depend on $D$ but not $N$. Our \cref{prop-swissArmy} in \cref{app-swissArmyProof} extends this result to the regularized case, $\lambda \geq 0$. Still, we are left with the fact that $C$ and $C'$ depend on $D$ in an unknown way.

In what follows, we consider both $\initNS$ and $\initIJ$, as they have complimentary strengths. Empirically, we find that $\initNS$ performs better in our LOOCV GLM experiments. But $\initIJ$ is computationally efficient beyond LOOCV and GLMs. E.g., for general models, computation of $\initNS$ requires inversion of a new Hessian for each $n$, whereas $\initIJ$ needs only the inversion of $H(\hat\theta)$ for all $n$. In terms of theory, $\initNS$ has a tighter error bound of $o(1/N)$ for GLMs. But the theory behind $\initIJ$ applies more generally, and, given a good bound on the gradients, may provide a tighter rate.

\section{Problems in high dimensions} \label{sec-highDProblems}
In the above discussion, we noted that there exists encouraging theory governing the behavior of $\initNS$ and $\initIJ$ when $D$ is fixed and $N$ grows large. We now describe issues with $\initNS$ and $\initIJ$ when $D$ is large relative to $N$. The first challenge for both approximations given large $D$ is computational. Since every variant of CV or approximate CV requires running the machine learning algorithm of interest at least once, we will focus on the cost of the approximations \emph{after} this single run. Given $\hat\theta$, both approximations require the inversion of a $D \times D$ matrix. Calculation of $\initIJ$ across $n \in [N]$ requires a single matrix inversion and $N$ matrix multiplications for a runtime in $O(D^3 + ND^2)$. In general, calculating $\initNS$ has runtime of $O(ND^3)$ due to needing an inversion for each $n$. In the case of GLMs, though, $\nabla^2_\theta f_n$ is a rank-one matrix, so standard rank-one updates give a runtime of $O(D^3 + ND^2)$ as well.

The second challenge for both approximations is the invertibility of $H(\hat\theta)$ and $H(\hat\theta) - (1/N)\nabla_\theta^2 f(x_n^T \theta, y_n)$ that was assumed in defining $\initNS$ and $\initIJ$. We note that, if $\nabla^2 R(\hat\theta)$ is only positive semidefinite, then invertibility of both matrices may be impossible when $D \geq N$; see \cref{app-invertibility} for more discussion.

The third and final challenge for both approximations is accuracy in high dimensions. Not only do existing error bounds behave poorly (or not exist) in high dimensions, but empirical performance degrades as well. To create \cref{fig-errorScalingExample}, we generated datasets from a sparse logistic regression model with $N$ ranging from 500 to 5,000. For the blue lines, we set $D=2$, and for the red lines we set $D=N/10$. In both cases, we see that error is much lower when $D$ is small and fixed.

 We recall that for large $N$ and small $D$, training error often provides a fine estimate of the out-of-sample error (e.g., see \citep{vapnik:1992:erm}). That is, CV is needed precisely in the high-dimensional regime, and this case is exactly where current approximations struggle both computationally and statistically. Thus, we wish to understand whether there are high-$D$ cases where approximate CV is useful. In what follows, we consider a number of options for tackling one or more of these issues and show that only one method is effective in high dimensions.

\textbf{Proposal \#1: Use randomized solvers to reduce computation.} Previously, \citet{koh:2017:influenceFunctions} have utilized $\initIJ$ for qualitative purposes, in which they are interested in its sign and relative magnitude across different $n$. They tackle the $O(D^3)$ scaling of $\initIJ$ by using the randomized solver from \citet{agarwal:2016:lissa}. While one might hope to replicate the success of \citet{koh:2017:influenceFunctions} in the context of approximate CV, we show in \cref{app-stochasticHessianNoGood} that this randomized solver performs poorly for approximating CV: while it can be faster than exactly solving the needed linear systems, it provides an approximation to exact CV that can be an order of magnitude less accurate.

\subsection{Sparsity via $\ell_1$ regularization.} Intuitively, if the exact $\thetan$'s have some low ``effective dimension'' $\deff \ll D$, we might expect approximate CV's accuracy to depend only on $\deff$. One way to achieve low $\deff$ is sparsity: i.e., we have $\hatdeff := \abs{\supp \hat\theta} \ll D$, where $\hat S := \supp \hat\theta$ collects the indices of the non-zero entries of $\hat\theta$. A way to achieve sparsity is choosing $R(\theta) = \| \theta\|_1$. However, note that $\initNS$ and $\initIJ$ cannot be applied directly in this case as $\| \theta\|_1$ is not twice-differentiable. \textbf{Proposal \#2}: \citet{rad:2018:detailedALOO,wang:2018:primalDualALOO} propose the use of a smoothed approximation to $\|\cdot\|_1$; however, as we show in \cref{sec-experiments}, this approach is often multiple orders of magnitude more inaccurate and slower than Proposal \#4 below.

\textbf{Proposal \#3: Subsample exact CV.} Another option is to bypass all the problems of approximate CV in high-$D$ by uniformly subsampling a small collection of LOOCV folds. This provides an unbiased estimate of exact CV and can be used with exact $\ell_1$ regularization. However, our experiments (\cref{sec-experiments}) show that, under a time budget, the results of this method are so variable that their error is often multiple orders of magnitude higher than Proposal \#4 below.

\textbf{Proposal \#4: Use the sparsity from $\hat\theta$.}
Instead, in what follows, we take the intuitive approach of approximating CV only on the dimensions in $\supp\hat\theta$. Unlike all previously discussed options, we show that this approximation is fast and accurate in high dimensions in both theory and practice.
For notation, let $X \in \R^{N \times D}$ be the covariate matrix, with rows $x_n$. For $S \subset [D]$, let $X_{\cdot, S}$ be the submatrix of $X$ with column indices in $S$; define $x_{nS}$ and $\theta_S$ similarly. Let $\dntwo := \left[ d^2 f(z, y_n) / dz^2 \right]_{z = x_n^T\hat\theta}$, and define the restricted Hessian evaluated at $\hat\theta$: $H_{\hat S \hat S} := X_{\cdot, \hat S}^T \mathrm{diag}\{ \dntwo \} X_{\cdot, \hat S}$. Further define the LOO restricted Hessian, $H^{\bn}_{\hat S \hat S} := H_{\hat S \hat S} - [\nabla^2_\theta f(x_n^T\hat\theta, y_n)]_{\hat S \hat S}$. Finally, without loss of generality, assume $\hat S = \{1, 2, \dots, \hatdeff \}$. We now define versions of $\initNS$ and $\initIJ$ restricted to the entries in $\supp\hat\theta$:
\begin{align}
&\NS := 
		\begin{pmatrix}
			\hat\theta_{\hat S} + (H^{\bn}_{\hat S \hat S})\inv \left[ \nabla_\theta f(x_n^T \hat\theta, y_n)\right]_{\hat S} \\ 	
			0
                \end{pmatrix} \\
&\IJ :=
		\begin{pmatrix}
			\hat\theta_{\hat S} + H_{\hat S \hat S}\inv \left[ \nabla_\theta f(x_n^T \hat\theta, y_n) \right]_{\hat S} \\
			0
		\end{pmatrix}.
	        \label{restrictedApproximations}
\end{align}
Other authors have previously considered $\NS$. \citet{rad:2018:detailedALOO,wang:2018:primalDualALOO} derive $\NS$ by considering a smooth approximation to $\ell_1$ and then taking the limit of $\initNS$ as the amount of smoothness goes to zero. In \cref{app-smoothing}, we show a similar argument can yield $\IJ$. Also, \citet{obuchi:2016:linearALOO,obuchi:2018:logisticALOO,beirami:2017:firstALOO} directly propose $\NS$ without using $\initNS$ as a starting point. We now show how $\NS$ and $\IJ$ avoid the three major high-dimensional challenges with $\initNS$ and $\initIJ$ we discussed above.

The first challenge was that compute time for $\initNS$ and $\initIJ$ scaled poorly with $D$. That $\NS$ and $\IJ$ do not share this issue is immediate from their definitions. 
\begin{prop} \label{prop-runtimes}
For general $f_n$, the time to compute $\NS$ or $\IJ$ scales with $\hatdeff$, rather than $D$. In particular, computing $\NS$ across all $n \in [N]$ takes $O(N \hatdeff^3)$ time, and computing $\IJ$ across all $n \in [N]$ takes $O(\hatdeff^3 + N \hatdeff^2)$ time. Furthermore, when $f_n$ takes the form of a GLM, computing $\NS$ across all $n \in [N]$ takes $O(\hatdeff^3 + N\hatdeff^2)$ time.
\end{prop}
%Comparing the $O(D^3)$ scaling of $\initNS$ and $\initIJ$ with the $O(\hatdeff^3)$ scaling of $\NS$ and $\IJ$, we have that the latter will be much faster to compute when $\hatdeff \ll D$. We note that this cubic scaling in dimension comes from assuming a standard Cholesky factorization based approach to solving the relevant linear systems. While one might consider randomized or approximate solvers to reduce the computational burden of $\initNS$ and $\initIJ$, we remark that this seems to be a complicated issue. For example, as noted in the introduction, we show in \cref{app-stochasticHessianNoGood} that while \citet{koh:2017:influenceFunctions} use a randomized solver to great success in a related context, this solver has poor accuracy for approximating CV. Thus, we believe \cref{prop-runtimes} represents the first description of approximations to CV being both fast \emph{and} -- as we will see below -- accurate.
%
The second high-dimensional challenge was that $H$ and $H^{\setminus n}$ may not be invertible when $D \geq N$. 
Notice the relevant matrices in $\NS$ and $\IJ$ are of dimension $\hatdeff = \abs{\hat S}$. So we need only make the much less restrictive assumption that $\hatdeff < N$, rather than $D < N$. We address the third and final challenge of accuracy in the next section.

% !TeX root=main_aistats.tex

\section{Approximation quality in high dimensions} \label{sec-theory}

Recall that the accuracy of $\initNS$ and $\initIJ$ in general has a poor dependence on dimension $D$. 
We now show that the accuracy of $\NS$ and $\IJ$ depends on (the hopefully small) $\hatdeff$ rather than $D$. We start by assuming a ``true'' population parameter\footnote{This assumption may not be necessary to prove the dependence of $\NS$ and $\IJ$ on $\hatdeff$, but it allows us to invoke existing $\ell_1$ support results in our proofs.} $\theta^* \in \R^D$ that minimizes the population-level loss, $\theta^* := \argmin \mathbb{E}_{x,y}[f(x^T\theta, y)]$, where the expectation is over $x,y$ from some population distribution. Assume $\theta^*$ is sparse with $S := \supp\theta^*$ and $\deff := |S|$. Our parameter estimate would be faster and more accurate if an oracle told us $S$ in advance and we worked just over $S$:
\begin{equation}
	\hat\phi := \argmin_{\phi \in \R^{\deff}} \frac{1}{N} \sum_{n=1}^N f(x_{nS}^T \phi, y_n) + \lambda \n{\phi}_1.
	\label{phiRestrictedOptimization}
\end{equation}
We define $\phin$ as the leave-one-out variant of $\hat\phi$ (as $\thetan$ is to $\hat\theta$). Let $\restrictedNS$ and $\restrictedIJ$ be the result of applying the approximation in $\NS$ or $\IJ$ to the restricted problem in \cref{phiRestrictedOptimization}; note that $\restrictedNS$ and $\restrictedIJ$ have accuracy that scales with the (small) dimension $\deff$.

Our analysis of the accuracy of $\NS$ and $\IJ$ will depend on the idea that if, for all $n$, $\NS$, $\IJ$, and $\thetan$ run over the same $\deff$-dimensional subspace, then the accuracy of $\NS$ and $\IJ$ must be identical to that of $\restrictedNS$ and $\restrictedIJ$. In the case of $\ell_1$ regularization, this idea specializes to the following condition, under which our main result in \cref{mainResult} will be immediate.
\begin{condition} \label{condition:supportStable}
  For all $n \in [N]$, we have $\supp\IJ = \supp\NS = \supp\thetan = S$.
\end{condition}
\begin{thm} \label{mainResult}
  Assume \cref{condition:supportStable} holds. Then for all $n$, $\thetan$ and $\IJ$ are (1) zero outside the dimensions $S$ and (2) equal to their restricted counterparts from \cref{phiRestrictedOptimization}:
\begin{align}
	&\thetan
		= \begin{pmatrix}
			\thetan[,S] \\
			0
		\end{pmatrix}
		= \begin{pmatrix}
			\phin \\
			0
		\end{pmatrix}, \;\; \nonumber \\
	&\IJ
		= \begin{pmatrix}
			\IJ[,S] \\
			0
		\end{pmatrix}
		= \begin{pmatrix}
			\restrictedIJ \\
			0
		\end{pmatrix}.
	\label{thetanSupport}
\end{align}
It follows that the error is the same in the full problem as in the low-dimensional restricted problem:
$\| \thetan - \IJ\|_2 = \| \phin - \restrictedIJ \|_2$. The same results hold for $\IJ$ and $\restrictedIJ$ replaced by $\NS$ and $\restrictedNS$.
\end{thm}
Taking \cref{condition:supportStable} as a given, \cref{mainResult} tells us that for $\ell_1$ regularized problems, $\IJ$ and $\NS$ inherit the fixed-dimensional accuracy of $\initIJ$ and $\initNS$ shown empirically in \cref{fig-errorScalingExample} and described theoretically in the references from \cref{introduction}. Taking a step further, one could show that $\IJ$ and $\NS$ are accurate for model assessment tasks by using results on the accuracy of exact CV for assessment (e.g., \citep{moustafa:2018:cvWorks,steinberger:2018:cvIntervalsTheory,barber:2019:jackknife+}).

Again, \cref{mainResult} is immediate if one is willing to assume \cref{condition:supportStable}, but when does \cref{condition:supportStable} hold? There exist assumptions in the $\ell_1$ literature under which $\supp\hat\theta = S$ \citep{lee:2014:generalSparseRecovery,li:2015:sparsistency}. If one took these assumptions to hold for all $\objn := (1/N)\sum_{m:\, m\neq n} f_m$, then \cref{condition:supportStable} would directly follow. However, it is not immediate that any models of interest meet such assumptions. Rather than taking such uninterpretable assumptions or just taking \cref{condition:supportStable} as an assumption directly, we will give a set of more interpretable assumptions under which \cref{condition:supportStable} holds.

In fact, we need just four principal assumptions in the case of linear and logistic regression; we conjecture that similar results hold for other GLMs. The first assumption arises from the intuition that, if individual data points are very extreme, the support will certainly change for some $n$. To avoid these extremes with high probability, we assume that the covariates follow a \emph{sub-Gaussian} distribution:
\begin{defn} \label{defineSubGaussian} [e.g., \citet{vershynin:2017:hdpBook}]
For $c_x > 0$, a random variable $V$ is $c_x$-\emph{sub-Gaussian} if
$
  	\mathbb{E}\left[ \exp\left(V^2 / c_x^2 \right) \right] \leq 2.
$
\end{defn}
\begin{assumption} \label{assum:randomX}
Each $x_n \in \R^D$ has zero-mean i.i.d.\ $c_x$-sub-Gaussian entries with $\mathbb{E}[x_{nd}^2] = 1$.
\end{assumption}
We conjecture that the unit-variance part of the assumption is unnecessary.
Conditions on the distributions of the responses $y_n$ will be specific to linear and logistic regression and will be given in \cref{assum:linearRegressionY,assum:logisticRegressionY}, respectively. Our results below will hold with high probability under these distributions. Note there are reasons to expect we cannot do better than high-probability results. In particular, \citet{xu:2012:sparsityStability} show that there exist worst-case training datasets for which sparsity-inducing methods like $\ell_1$ regularization are not stable as each datapoint is left out.

Our second principal assumption is an \emph{incoherence} condition.
\begin{assumption} \label{assum:incoherenceAssumption}
  The incoherence condition holds with high probability over the full dataset:
  \begin{equation*} \Pr\left[ \n{\nabla F(\theta^*)_{S^c, S} \left( \nabla^2 F(\theta^*)_{SS}\right)\inv}_{\infty} < 1 - \alpha \right] \leq e^{-25}, \end{equation*}
\end{assumption}
Authors in the $\ell_1$ literature often assume that incoherence holds deterministically for a given design matrix $X$ -- starting from the introduction of incoherence by \citet{zhao:2006:incoherence} and continuing in more recent work \citep{lee:2014:generalSparseRecovery, li:2015:sparsistency}. Similarly, we will take our high probability version in \cref{assum:incoherenceAssumption} as given. But we note that \cref{assum:incoherenceAssumption} is at least known to hold for the case of linear regression with an i.i.d.\ Gaussian design matrix (e.g., see Exercise 11.5 of \citet{hastie:2015:sls}). We next place some restrictions on how quickly $D$ and $\deff$ grow as functions of $N$.
\begin{assumption} \label{assum:DDeffScaling}
  As functions of $N$, $D$ and $\deff$ satisfy: (1) $D = o(e^N)$, (2) $\deff = o([N / \log N]^{2/5})$, and (3) $\deff^{3/2} \sqrt{\log D} = o(N)$.
\end{assumption}
The constraints on $D$ here are particularly loose. While those on $\deff$ are tighter, we still allow polynomial growth of $\deff$ in $N$ for some lower powers of $N$. Our final assumption is on the smallest entry of $\theta^*_S$. Such conditions -- typically called \emph{beta-min conditions} -- are frequently used in the $\ell_1$ literature to ensure $\hat S = S$ \citep{wainwright:2009:linearRegressionLasso,lee:2014:generalSparseRecovery,li:2015:sparsistency}.
\begin{assumption} \label{assum:thetaMin}
$\theta^*$ satisfies
$
	\min_{s \in S} \abs{\theta^*_s} > \sqrt{\deff} T_{min} \lambda,
$
where $T_{min}$ is some constant relating to the objective function $f$; see \cref{assum:thetaMin-app} in \cref{app-liAssumptions} for an exact description.
\end{assumption}
\subsection{Linear regression} \label{sec-linearRegression}

We now give the distributional assumption on the responses $y_n$ in the case of linear regression and then show that \cref{condition:supportStable} holds.
\begin{assumption} \label{assum:linearRegressionY}
$\forall n, y_n = x_n^T\theta^* + \eps_n$, where the $\eps_n$ are i.i.d.\ $c_\eps$-sub-Gaussian random variables.
\end{assumption}
\begin{thm}[Linear Regression] \label{linearRegressionTheorem}
  Take \cref{assum:randomX,assum:incoherenceAssumption,assum:thetaMin,assum:linearRegressionY,assum:DDeffScaling}. Suppose the regularization parameter $\lambda$ satisfies
  \begin{align}
    \lambda \geq \frac{C}{\alpha - \MJLin} & \left( \sqrt{ \frac{c_x^2 c_\eps^2 \log D}{N} + \frac{25c_x^2 c_\eps^2}{N}}  \right. \nonumber \\
    & \left. \quad + \frac{4c_x c_\eps (\log(ND) + 26)}{N}\right)   \label{linearRegressionLambda},
   \end{align}
  where $C > 0$ is a constant in $N, D, \deff ,c_x$, and $c_\eps$, and $\MJLin$ is a scalar given by \cref{MJDef} in \cref{app-proofs} that satisfies, as $N \to \infty$, $\MJLin = o(1)$. Then for $N$ sufficiently large, \cref{condition:supportStable} holds with probability at least $1 - 26e^{-25}$.
\end{thm}
A full statement and proof of \cref{linearRegressionTheorem}, including the exact value of $\MJLin$, appears in \cref{app-proofs}.
A corollary of \cref{mainResult} and \cref{linearRegressionTheorem} together is that, under \cref{assum:thetaMin,assum:randomX,assum:incoherenceAssumption,assum:DDeffScaling,assum:linearRegressionY}, the LOOCV approximations $\IJ$ and $\NS$ have accuracy that depends on (the ideally small) $\deff$ rather than (the potentially large) $D$.

It is worth considering how the allowed values of $\lambda$ in \cref{linearRegressionLambda} compare to previous results in the $\ell_1$ literature for the support recovery of $\hat\theta$. We will talk about a sequence of choices of $\lambda$ scaling with $N$ denoted by $\lambda_N$. Theorem 11.3 of \citet{hastie:2015:sls} shows that $\lambda_N \geq c \sqrt{\log (D) / N}$ (for some constant $c$ in $D$ and $N$) is sufficient for ensuring that $\supp\hat\theta \subseteq S$ with high probability in the case of linear regression. Thus, we ought to set $\lambda_N \geq c \sqrt{\log(D) / N}$ to ensure support recovery of $\hat\theta$. Compare this constraint on $\lambda_N$ to the constraint implied by \cref{linearRegressionLambda}. We have that $\MJLin = o(1)$ as $N\to\infty$, so that, for large $N$, the bound in \cref{linearRegressionLambda} becomes $\lambda_N \geq c' \sqrt{\log(D) / N}$ for some constant $c'$. Thus, the sequence of $\lambda_N$ satisfying \cref{linearRegressionLambda} scales at exactly the same rate as those that ensure $\supp\hat\theta \subseteq S$. The scaling of $\lambda_N$ is important, as the error in $\hat\theta$, $\| \hat\theta - \theta^*\|_2^2$, is typically proportional to $\lambda_N$. The fact that we have not increased the asymptotic scaling of $\lambda_N$ therefore means that we can enjoy the same decay of $\|\hat\theta - \theta^*\|_2^2$ while ensuring $\supp\thetan = S$ for all $n$.

\subsection{Logistic regression} \label{sec-logisticRegression}
We now give the distributional assumption on the responses $y_n$ in the case of logistic regression.
\begin{assumption} \label{assum:logisticRegressionY}
$\forall n,$ we have $y_n \in \set{\pm 1}$ with $\Pr\left[y_n = 1\right] = 1/(1 + e^{-x_n^T\theta^*})$.
\end{assumption}
We will also need a condition on the minimum eigenvalue of the Hessian.
\begin{assumption} \label{assum:logisticRegressionLambdaMin}
Assume for some scalar $L_{min}$ that may depend on $N, \deff,$ and $c_x$, we have
$$ \Pr\left[ \lambda_{min} \left( \nabla_\theta^2 F(\theta^*)_{SS} \right) \leq L_{min} \right] \leq e^{-25}. $$
Furthermore, assume the scaling of $L_{min}$ in $N$ and $\deff$ is such that, under \cref{assum:DDeffScaling} and for sufficiently large $N$, $L_{min} \geq C N$ for some constant $C$ that may depend on $c_x$.
\end{assumption}
In the case of linear regression, we did not need an analogue of \cref{assum:logisticRegressionLambdaMin}, as standard matrix concentration results tell us that its Hessian satisfies \cref{assum:logisticRegressionLambdaMin} with $L_{min} = N - Cc_x^2 \sqrt{N\deff}$ (see \cref{linearRegressionLambdaMin} in \cref{app-proofs}). The Hessian for logistic regression is significantly more complicated, and it is typical in the $\ell_1$ literature to make some kind of assumption about its eigenvalues \citep{bach:2010:concordantLogistic,li:2015:sparsistency}. Empirically, \cref{assum:logisticRegressionLambdaMin} is satisfied when \cref{assum:randomX,assum:logisticRegressionY} hold; however we are unaware of any results in the literature showing this is the case.
\begin{thm}[Logistic Regression] \label{logisticRegressionTheorem}
  Take \cref{assum:randomX,assum:incoherenceAssumption,assum:thetaMin,assum:logisticRegressionY,assum:DDeffScaling,assum:logisticRegressionLambdaMin}. Suppose the regularization parameter $\lambda$ satisfies:
  \begin{align}
    \lambda \geq \frac{C}{\alpha - \MJLogr} & \left( \sqrt{c_x^2 \frac{25 + \log D}{N}} \right. \nonumber \\
    & \quad \left. + \frac{\sqrt{2c_x^2 \log(ND)} + \sqrt{50c_x^2}}{N} \right),
    %\lambda & \geq \frac{1}{\alpha-\MJLogr} \left( \sqrt{\frac{25 + \log D}{NC}} +   \frac{\sqrt{2c_x \log(ND)} + \sqrt{50c_x}}{N} \right)
   \label{logisticRegressionLambda}
   \end{align}
  where $C,C'$ are constants in $N, D, \deff,$ and $c_x$, and $\MJLogr$ is a scalar given by \cref{MJLogr}, that, as $N\to\infty$, satisfies $\MJLogr = o(1)$. Then for $N$ sufficiently large, \cref{condition:supportStable} is satisfied with probability at least $1 - 43e^{-25}$.
\end{thm}
A restatement and proof of \cref{logisticRegressionTheorem} are given as \cref{logisticRegressionTheorem-app} in \cref{app-proofs}. Similar to the remarks after \cref{linearRegressionTheorem}, \cref{logisticRegressionTheorem} implies that when applied to logistic regression, $\IJ$ and $\NS$ have accuracy that depends on (the ideally small) $\deff$ rather than (the potentially large) $D$, even when $D = o(e^N)$. 

\cref{logisticRegressionTheorem} has implications for the work of \citet{obuchi:2018:logisticALOO}, who conjecture that, as $N\to\infty$, the change in support of $\ell_1$ regularized logistic regression becomes negligible as each datapoint is left out; this assumption is used to derive a version of $\NS$ for logistic regression. Our \cref{logisticRegressionTheorem} confirms this conjecture by proving the stronger fact that the support is unchanged with high probability for finite data.

\begin{figure}
    \centering
    \includegraphics[scale=.4]{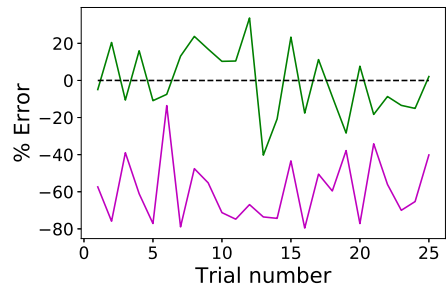}
    \caption{Error (\cref{percentError}) across approximations for $\ell_1$ LOOCV (legend shared with \cref{manyNtrialsTimings}). The error for $\IJ$ (black dashed) is too small to see, but nonzero; it varies between $-0.06\%$ and $0.04\%$.} \label{manyNtrialsAccuracy}
\end{figure}

\begin{figure}
  \centering
  \includegraphics[scale=.4]{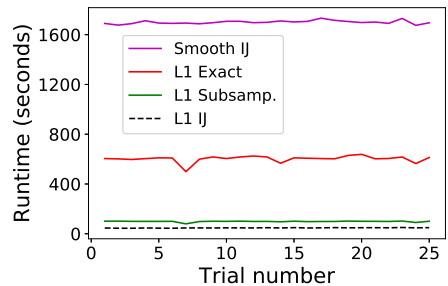}
  \caption{Runtimes for the experiments in \cref{manyNtrialsAccuracy} with exact CV (red) included for comparison. The $D \times D$ matrix inversion in the smoothed problem is so slow that even exact CV with an efficient $\ell_1$ solver is faster.} \label{manyNtrialsTimings}
\end{figure}

\begin{figure}
  \centering
    \includegraphics[scale=0.4]{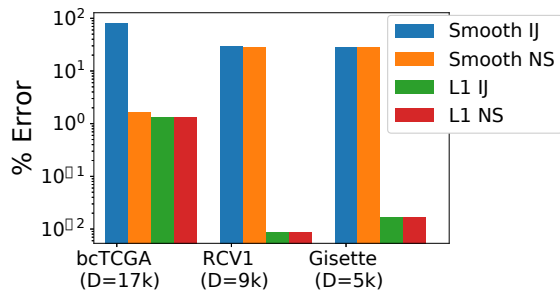}

\caption{Log percent accuracy (\cref{percentError}) for real data experiments. For each dataset, we give the accuracy of approximate CV compared to exact CV using both a smoothed approximation to $\ell_1$ and the $\IJ$, $\NS$ approximations. For the bcTCGA dataset (linear regression), the nearly quadratic objective seems to be extremely well approximated by one Newton step, making $\initNS[R^\eta]$ significantly more accurate than $\initIJ[R^\eta]$; see the note at the end of \cref{app-modifiedDerivation} about the exactness of $\initNS$ on quadratic objectives. } \label{realExperimentsAccuracy}
\end{figure}

\section{Experiments} \label{sec-experiments}

We now empirically verify the good behavior of $\NS$ and $\IJ$ (i.e.\ proposal \#4) and show that it far outperforms \#2 (smoothing $\ell_1$) and \#3 (subsampling) in our high-dimensional regime of interest. All the code to run our experiments here is available online.\footnote{\url{https://bitbucket.org/wtstephe/sparse_appx_cv/}} We focus comparisons in this section on proposals \#2--\#4, as they all directly address $\ell_1$-regularized problems. For an illustration of the failings of proposal \#1, see \cref{app-stochasticHessianNoGood}. To illustrate \#2, we consider the smooth approximation given by \citet{rad:2018:detailedALOO}:
$
R^\eta(\theta) := \sum_{d=1}^D \frac{1}{\eta}\big( \log(1 + e^{\eta\theta_d}) + \log(1 + e^{-\eta\theta_d}) \big).
$
While $\lim_{\eta\to\infty} R^\eta(\theta) = \n{\theta}_1$, we found that this approximation became numerically unstable for %the purposes of
optimization when $\eta$ was much larger than 100, so we set $\eta = 100$ in our experiments. 

\textbf{Simulated experiments.}  First, we trained logistic regression models on twenty-five random datasets in which $x_{nd} \overset{i.i.d.}{\sim} N(0,1)$ with $N=500$ and $D=40{,}000$. We set $\lambda = 1.5\sqrt{\log(D)/N}$ to mimic our condition in \cref{logisticRegressionLambda}. The true $\theta^*$ was supported on its first five entries. We evaluate our approximations by comparing the CV estimate of out-of-sample error (``$\mathrm{LOO}$'') to the approximation $\mathrm{ALOO} := \frac{1}{N} \sum_{n=1}^N f(x_n^T \IJ, y_n).$ We report percent error:
\begin{equation}
\abs{\mathrm{ALOO} - \mathrm{LOO}} / \mathrm{LOO}. \label{percentError}
\end{equation}
%
%where $ \mathrm{ALOO} = \frac{1}{N} \sum_{n=1}^N f(x_n^T \IJ, y_n).$ The exact CV estimate of out-of-sample error, $\mathrm{LOO}$, is computed similarly.

\cref{manyNtrialsAccuracy} compares the accuracy and run times of proposals \#2 and \#3 versus $\IJ$. We chose the number of subsamples so that subsampling CV would have about the same runtime as computing $\IJ$ for all $n$.\footnote{Specifically, we computed 41 different $\thetan$ for each trial in order to roughly match the time cost of computing $\IJ$ for all $N=500$ datapoints.} We see that subsampling usually has much worse accuracy than $\IJ$. Using $\initIJ$ with $R^{100}(\theta)$ as a regularizer is even worse, as we approximate over all $D$ dimensions; the resulting approximation is slower and less accurate -- by multiple orders of magnitude -- across all trials.

\paragraph{The importance of setting $\lambda$.} Our theoretical results heavily depend on particular settings of $\lambda$ to obtain the fixed-dimensional error scaling shown in blue in \cref{fig-errorScalingExample}. One might wonder if such a condition on $\lambda$ is necessary for approximate CV to be accurate. We offer evidence in \cref{app-supportRecoveryExperiments} that this scaling is necessary by empirically showing that when $\lambda$ violates our condition, the error in $\IJ$ grows with $N$.

\begin{figure}
  \centering
  \includegraphics[scale=0.4]{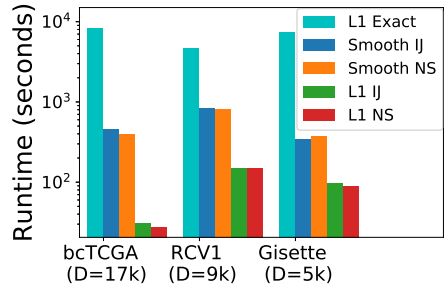}
  \caption{Log runtimes for experiments in \cref{realExperimentsAccuracy}, with exact CV included for comparison.} \label{realExperimentsRuntime}
\end{figure}

%\begin{tabular}{c c c | c c}
%  Dataset & Model & $N$; $D$ & \multicolumn{2}{c}{Avg. Error}\\
%  \hline
%  & & & $\ell_2$ & $\ell_1$ \\
%  \hline
  %bcTCGA & Linear & 536; 5,000 & 7.4e-3 & 2.6e-3 \\
  %Gisette & Logistic & 6,000; 4,955 & 8.6e-4 & 1.6e-4 \\
  %P53 & Logistic & 10,000 ; 5,409 & 3.3e-4 & 2.4e-5
%  bcTCGA & Linear & 536; 10,000 & 9.7e-1 & 9.7e-2 \\
%  Gisette & Logistic & 6,000; 4,955 & 1.1e-1 & 5.0e-2 \\
%  P53 & Logistic & 10,000 ; 5,409 & 1.6e-2 & 1.5e-3
%\end{tabular}

%\begin{figure}
%\centering
%\begin{tabular}{c | c c c c}
%  Dataset & \multicolumn{2}{c}{Time (exact/approx. CV, in minutes)} \\
%  \hline
%  & $\ell_2$ & $\ell_1$ \\
%  \hline
%  bcTCGA & 210.6 / 0.7  &  183.4 / 0.6 \\
%  Gisette & 661.1 / 4.2  &  102.1 / 4.1 \\
%  P53 & 1,232.8 / 6.8 & 93.0 / 5.6 \\
%\end{tabular}
%\caption{} \label{realExperimentsTimings}
%\end{figure}

\textbf{Real data experiments.} %\label{sec-realExperiments}
We next study how dependent our results are on the particular distributional assumptions \cref{linearRegressionTheorem,logisticRegressionTheorem}. We explore this question with a number of publicly available datasets \citep{bcTCGA, rcv1, gisette}. We chose these datasets because they have a high enough dimension to observe the effect of our results, yet are not so large that running exact CV for comparison is prohibitively expensive; see \cref{app-realExperiments} for details (including our settings of $\lambda$). For each dataset, we approximate CV for the $\ell_1$ regularized model using $\IJ$ and $\NS$. For comparison, we report the accuracy of $\initIJ[R^\eta]$ and $\initNS[R^\eta]$ with $\eta = 100$. Our results in \cref{realExperimentsAccuracy} show that $\IJ$ is significantly faster and more accurate than exact CV or smoothing.

To demonstrate the scalability of our approximations, we re-ran our RCV1 experiment on a larger version of the dataset with $N=20{,}242$ and $D = 30{,}000$. Based on the time to compute exact LOOCV for twenty datapoints, we estimate exact LOOCV would have taken over two weeks to complete, whereas computing both $\NS$ and $\IJ$ for \emph{all} $n$ took three minutes.
%To briefly demonstrate the scalability of our approximations, we fit a $\ell_1$-regularized logistic regression model to a larger version of the RCV1 dataset with $N=20,242$ and $D = 30,000$. Based on the time to compute exact LOOCV for twenty datapoints, we estimate exact LOOCV would have taken over two weeks to complete, whereas computing both $\NS$ and $\IJ$ for \emph{all} $n$ took three minutes.

%\textbf{Hyperparameter selection.} This paper is entirely focused on CV for model assessment, rather than model selection. In \cref{app-lambdaSelection}, we demonstrate that approximate CV for hyperparameter tuning can have complex and unexpected behavior when used for hyperparameter tuning. We leave further investigation of this subject to future work.

\section{Conclusions and future work}

We have provided the first analysis of when CV can be approximated quickly \emph{and} accurately in high dimensions with guarantees on quality. We have seen that, out of a number of proposals in the literature, running approximate CV on the recovered support (i.e., $\NS$ and $\IJ$) forms the only proposal that reaches these goals both theoretically and empirically. We hope this analysis will serve as a starting point for further understanding of when approximate CV methods work for high-dimensional problems.
\\\\
We see three interesting directions for future work. First, this work has focused entirely on approximate CV for model assessment. In \cref{app-lambdaSelection}, we show that approximate CV for model \emph{selection} can have unexpected and undesirable behavior; we believe understanding this behavior is one of the most important future directions in this area. Second, one could extend our results to results to the higher order infinitesimal jackknife presented in \citet{giordano:2019:HOIJ}. Finally, it would be interesting to consider our approximations as a starting point for subsampling estimators, as proposed in \citet{magnusson:2019:scalableAki}.

%We have shown that in both theory and practice, we can expect $\IJ$ and $\NS$ to harness the low effective dimension present in $\ell_1$ regularized GLMs to give significant improvements in both computation time and approximation accuracy. However, these results, as well as those in all previous works on approximate CV, only apply to assessing the generalization error for a single fixed model (i.e., for a fixed $\lambda$). These analyses do not address using CV for model selection, where one fits with many values of $\lambda$ and selects the one with the best generalization error. In \cref{app-lambdaSelection}, we give some empirical evidence that approximate CV methods are sometimes -- but not always -- suitable for this purpose. Specifically, we show that the infinitesimal jackknife approximation sometimes selects $\lambda = 0$. Fortunately, this failure mode is easy to detect, suggesting the following workflow: run approximate CV to select $\lambda$, and in the event that $\lambda = 0$ is selected, run exact CV instead. Still, this suggestion is purely based on empirics. We believe that developing a practically applicable theory of the uses and limitations of approximate CV for selecting $\lambda$ is one of the most important directions for future work in this area.

\subsection*{Acknowledgements}

This research is
supported in part by DARPA, the CSAIL-MSR Trustworthy AI Initiative,
an NSF CAREER Award, an ARO YIP Award, and ONR.

\bibliography{references}
\bibliographystyle{plainnat}

\newpage
\appendix
\onecolumn

\section{Cross-validation methods} \label{app-CVReview}

In this appendix, we review standard cross-validation (CV) for optimization problems of the form:
$$ \argmin_{\theta \in \Theta} \sum_{n=1}^N f_n(\theta) + \lambda R(\theta),$$
where $\Theta \subseteq \R^D$. By leave-one-out cross-validation (LOOCV), we mean the process of repeatedly computing:
$$ \thetan := \sum_{m:\, m\neq n} f_m(\theta) + \lambda R(\theta).$$
The parameter estimates $\{\thetan\}_{n=1}^N$ can then be used to produce an estimate of variability or out-of-sample error; e.g., to estimate the out-of-sample error, one computes $(1/N) \sum_n f_n(\thetan)$. By $K$-fold cross-validation, we mean the process of splitting up the dataset into $K$ disjoint folds, $S_1, \dots, S_K$ with $S_1 \cup \dots \cup S_K = [N]$. One then estimates the parameters:
$$ \hat\theta_{\backslash S_k} := \argmin_{\theta\in\Theta} \sum_{n:\, n \not\in S_k} f_n(\theta) + \lambda R(\theta).$$
The parameter estimates $\{\hat\theta_{\backslash S_k}\}_{k=1}^K$ can then be used to produce an estimate of variability or out-of-sample error.

\section{Scaling of the leave-one-out objective} \label{app-LOOScaling}

We defined $\thetan$ as the solution to the following optimization problem:
$$ \thetan := \argmin_{\theta\in\Theta} \frac{1}{N} \sum_{m:\, m\neq n} f_m(\theta) + \lambda R(\theta).$$
An alternative would be to use the objective $1/(N-1) \sum_{m:\, m \neq n} f_m + \lambda R$ in order to keep the scaling between the regularizer and the objective the same as in the full-data problem. Indeed, all existing theory that we are aware of for CV applied to $\ell_1$ regularized problems seems to follow the $1/(N-1)$ scaling \citep{homrighausen:2014:l1LOO,homrighausen:2013:l1Kfold, miolane:2018:uniformControl, chetverikov:2019:l1Kfold}. On the other hand, all existing approaches to approximate LOOCV for regularized problems have used the $1/N$ scaling that we have given \citep{beirami:2017:firstALOO,rad:2018:detailedALOO,wang:2018:primalDualALOO,xu:2019:ALOOAMP,obuchi:2016:linearALOO,obuchi:2018:logisticALOO}. Note that the scaling is not relevant in \citet{giordano:2018:ourALOO}, as they do not consider the regularized case. As our work is aimed at identifying when existing approximations work well in high dimensions, we have followed the $1/N$ choice from the literature on approximate LOOCV. The different results from using the two scalings may be insignificant when leaving only one datapoint out. But one might expect the difference to be substantial for, e.g., $K$-fold CV. We leave an understanding of what the effect of this scaling is (if any) to future work.

\section{Approximately solving $\initIJ$ and $\initNS$} \label{app-stochasticHessianNoGood}

\begin{figure}
  \centering
  \begin{tabular}{cc}
    \includegraphics[scale=0.5]{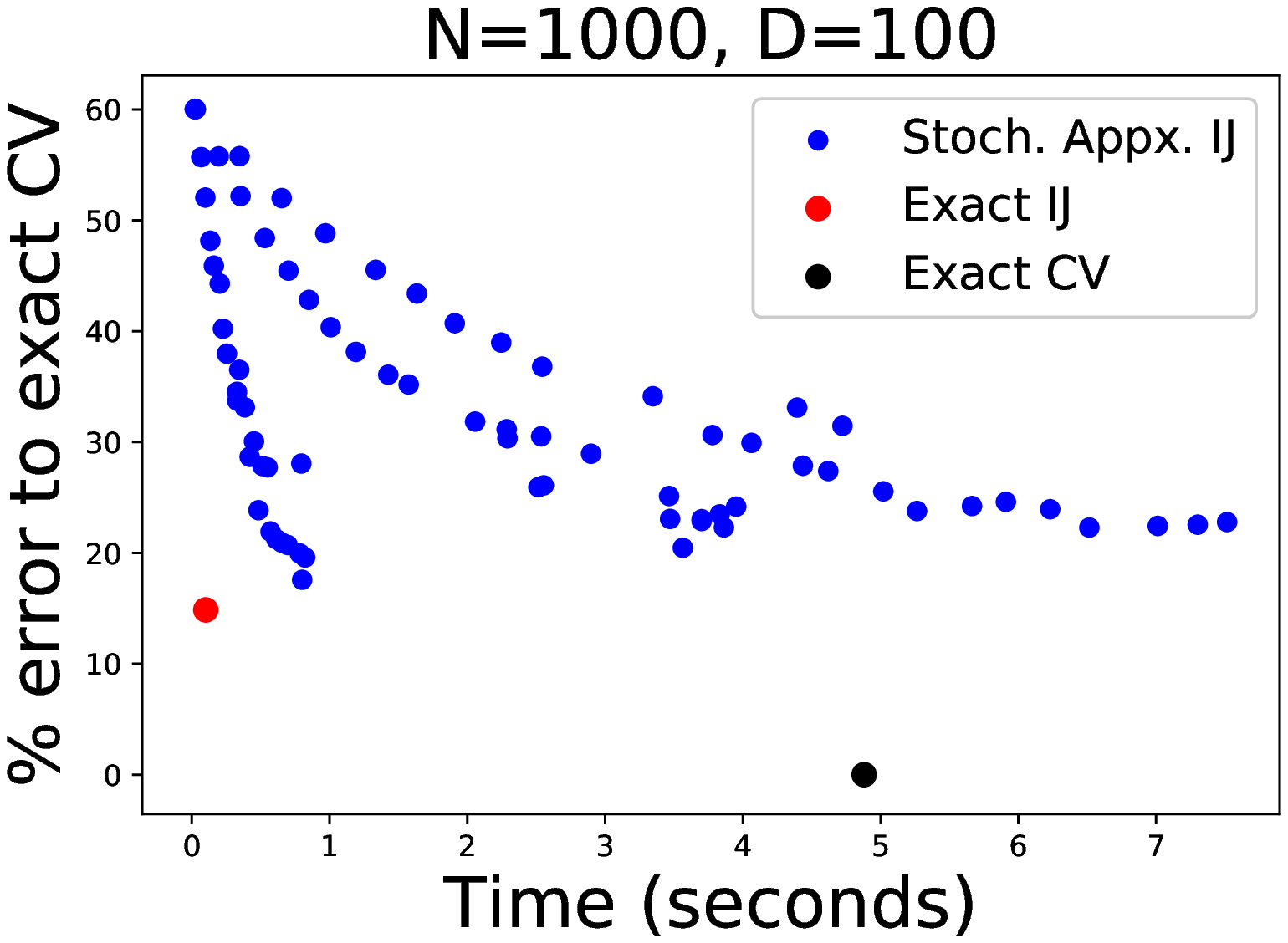} &
    \includegraphics[scale=0.5]{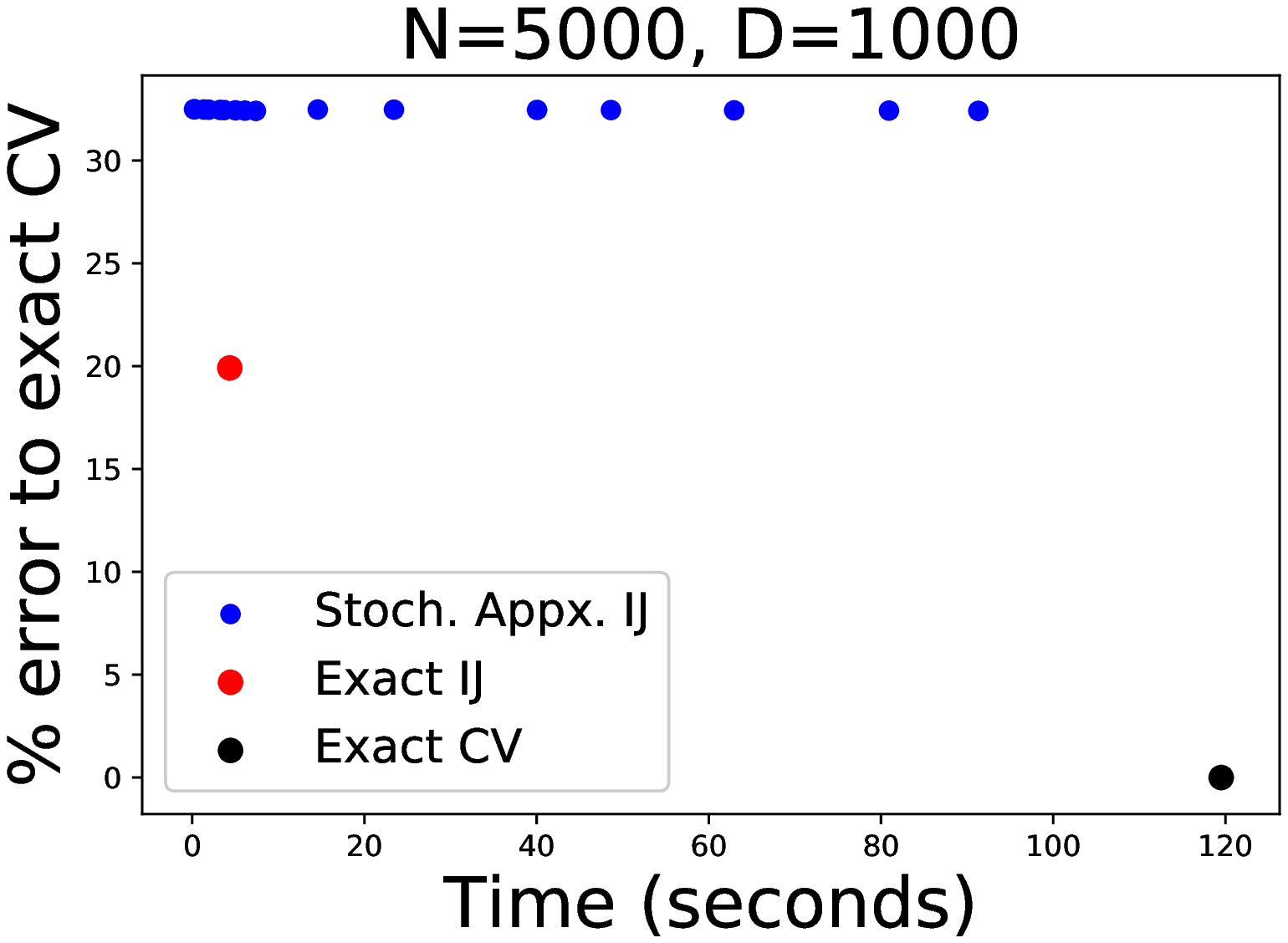}
  \end{tabular}
\caption{Stochastic Hessian experiments from \cref{app-stochasticHessianNoGood}. We show percent error of approximation versus compute time for two different dataset sizes. We show three techniques for computing CV: exactly computing CV (black dot, which naturally has 0\% error), $\initIJ$ with exactly computing the needed linear systems (red dot), and $\initIJ$ with the stochastic solves described in \cref{app-stochasticHessianNoGood} (blue dots, one for each setting of the parameters $K$ and $M$). Values of $M$ and $K$ used are described in \cref{app-stochasticHessianNoGood}; we use an extended range of settings for the smaller dataset to more extensively illustrate the approximation's behavior. Settings of $K$ and $M$ for which the stochastic solves are roughly as fast as exactly computing $\initIJ$ result in a significantly less accurate approximation of CV.}
\label{stochasticHessianNoGood}
\end{figure}

We have seen $\initIJ$ and $\initNS$ are in general not accurate for high-dimensional problems. Even worse, they can become prohibitively costly to compute due to the $O(D^3)$ cost required to solve the needed linear systems. One idea to at least alleviate this computational burden, proposed by \citet{koh:2017:influenceFunctions} in a slightly different context, is to use a stochastic inverse Hessian-vector-product from \citet{agarwal:2016:lissa} to approximately compute $\initIJ$ and $\initNS$. Although this method works well for the purposes of \citet{koh:2017:influenceFunctions}, we will see that in the context of approximate CV, it adds a large amount of extra error on top of the already inaccurate $\initNS$ and $\initIJ$. 

We first describe this stochastic inverse Hessian-vector-product technique and argue that it is not suitable for approximating cross-validation. The main idea from \citet{agarwal:2016:lissa} is to use the series:
$$ H\inv = \sum_{k=0}^\infty (I - H)^k ,$$
which holds for any positive definite $H$ with $\|H\|_{op} \leq 1$. Now, we can both truncate this series at some level $K$ and write it recursively as:
$$ H\inv \approx H\inv_K := I + (I - H)H\inv_{K-1},$$
where $H\inv_0 = I$. Next, to avoid computing $H$ explicitly, we can note that if $A_k$ is some random variable with $E[A_k] = H$, we can instead just sample a new $A_k$ at each iteration to define:
$$ \bar H\inv_k := I + (I - A_k) \bar H\inv_{K-1} .$$
In our case, we pick a random $n_k \in [N]$ and set $A_k = \nabla^2_\theta f(x_{n_k}^T \theta, y_{n_k}) + (1/N) \lambda \nabla^2 R(\theta)$. Finally, \citet{agarwal:2016:lissa} suggest taking $M$ samples of $\bar H\inv_K$ and averaging the results to lower the variance of the estimator. This leaves us with two parameters to tune: $M$ and $K$. Increasing either will make the estimate more accurate and more expensive to compute. \citet{koh:2017:influenceFunctions} use this approximation to compute $\initIJ$ for high dimensional models such as neural networks; however, we remark that their interest lies in the qualitative properties of $\initIJ$, such as signs and relative magnitudes across various values of $n$. It remains to be seen whether this stochastic solver can be successfully used to approximate CV.

To test the application of this approximation to approximate CV, we generated a synthetic logistic regression dataset with covariates $x_{nd} \overset{i.i.d.}{\sim} N(0,1)$. We use $R(\theta) = \| \theta \|_2^2$. In \cref{stochasticHessianNoGood}, we show that for two settings of $N$ and $D$ there are no settings of $M$ and $K$ for which using $\bar H_K\inv$ to compute $\initIJ$ provides a both fast and accurate approximation to CV. Specifically, we range $K \in \set{1,20,30,50,60,80,100,120}$ and $M \in \set{2,25}$ and see that when the stochastic approximation is faster, it provides only a marginal speedup while providing a significantly worse approximation error.

\section{Further details of \cref{modifiedApproximation} and \cref{regularApproximation}} \label{app-approximation}

In \cref{sec-approximation}, we briefly outlined the approximations $\initNS$ and $\initIJ$ to $\thetan$; we give more details about these approximations and their derivations here. Recall that we defined $H(\hat\theta) := (1/N) \sum_{n=1}^N \nabla_\theta^2 f(x_n^T\hat\theta, y_n) + \lambda \nabla^2_\theta R(\hat\theta)$. We first restate the ``infinitesimal jackknife'' approximation from the main text, which was derived by the same approach taken by \citet{giordano:2018:ourALOO}:

\begin{equation} \thetan \approx \initIJ := \hat\theta + \frac{1}{N}H(\hat\theta)\inv \nabla_\theta f(x_n^T\hat\theta, y_n). \label{app-regularApproximation} \end{equation}

The ``Newton step'' approximation, similar to the approach in \citet{beirami:2017:firstALOO} and identical to the approximation in \citet{rad:2018:detailedALOO, wang:2018:primalDualALOO}, is:
\begin{equation} \thetan \approx \initNS := \hat\theta + \frac{1}{N}\bigg( H(\hat\theta) - \frac{1}{N} \nabla^2_\theta f(x_n^T\hat\theta, y_n) \bigg) \inv \nabla_\theta f(x_n^T\hat\theta, y_n). \label{app-modifiedApproximation} \end{equation}

\subsection{Derivation of $\initIJ$} \label{app-regularDerivation}

We will see in \cref{app-swissArmyProof} that, after some creative algebra, $\initIJ$ is an instance of $\hat\theta_{IJ}$ from Definition 2 of \citet{giordano:2018:ourALOO}. However, this somewhat obscures the motivation for considering \cref{app-regularApproximation}. As an alternative to jamming our problem setup into that considered by \citet{giordano:2018:ourALOO}, we can more directly obtain the approximation in \cref{app-regularApproximation} by a derivation only slightly different from that in \citet{giordano:2018:ourALOO}. We begin by defining $\thetaw$ as the solution to a weighted optimization problem with weights $w_n \in \R$:
\begin{equation} \thetaw := \argmin_{\theta \in \Theta} G(w, \theta) :=\argmin_{\theta \in \Theta} \frac{1}{N} \sum_{n=1}^N w_n f(x_n^T\theta, y_n) + \lambda R(\theta), \label{weightedOptProblem} \end{equation}
where we assume $G$ to be twice continuously differentiable with an invertible Hessian at $\thetaone$ (where $\thetaone$ is the solution in \cref{weightedOptProblem} with all $w_n=1$). For example, we have that $\thetan = \thetaw$ if $w$ is the $N$-dimensional vector of all ones but with a zero in the $n$th coordinate. We will form a linear approximation to $\thetaw$ as a function of $w$. To do so, we will need to compute the derivatives $d\thetaw / dw_n$ for each $n$. To compute these derivatives, we begin with the first order optimality condition of \cref{weightedOptProblem} and take a total derivative with respect to $w_n$:
\begin{align*}
&\p{G}{\theta}\at{w=1, \theta=\thetaone} = 0 \\
&\implies \frac{d}{dw_n} \p{G}{\theta}\at{w=1, \theta=\thetaone} = \frac{\partial^2 G}{\partial \theta \partial w_n}\at{w=1,\theta=\thetaone} \frac{dw_n}{dw_n} + \frac{\partial^2 G}{\partial\theta^2} \at{w=1,\theta=\thetaone} \frac{d\thetaw}{d w_n}\at{w=1} = 0.
\end{align*}
Re-arranging, defining $H(\thetaone) := \nabla^2_\theta G(w, \thetaone)$,  and using the assumed invertibility of $H(\thetaone)$ gives:
\begin{equation} \frac{d\hat\theta}{dw_n}\at{w = 1} = -\left(\frac{\partial^2 G}{\partial\theta^2}\at{w=1, \theta=\thetaone} \right)\inv \frac{\partial^2 G}{\partial w_n \partial\theta}\at{w = 1, \theta=\thetaone} = -\frac{1}{N}H(\hat\theta)\inv \nabla_\theta f(x_n^T \hat\theta, y_n). \end{equation}
In the final equality, we used the fact that $\thetaone = \hat\theta$. Now, by a first order Taylor expansion around $w = (1, 1, \dots, 1)$, we can write:
\begin{align}
  \thetaw &\approx \hat\theta + \sum_{n=1}^N \frac{d\hat\theta}{dw_n}\at{w = 1} (w_n - 1) \\
  & = \hat\theta - \frac{1}{N} \sum_{n=1}^N H(\hat\theta)\inv \nabla_\theta f(x_n^T \hat\theta, y_n) (w_n - 1).
\end{align}
For the special case of $w$ being the vector of all ones with a zero in the $n$th coordinate (i.e., the weighting for LOOCV), we recover \cref{app-regularApproximation}.

\subsection{Invertibility in the definition of $\initNS$ and $\initIJ$} \label{app-invertibility}

In writing \cref{modifiedApproximation,regularApproximation} we have assumed the invertibility of $H(\hat\theta)$ and $H(\hat\theta) - (1/N)\nabla_\theta f(x_n^T\theta, y_n)$. We here note a number of common cases where this invertibility holds. First, if $\nabla^2 R$ is positive definite for all $\theta$ (as in the case of $R = \| \cdot\|_2^2$), then these matrices are always invertible. If $R$ is merely convex, $H(\hat\theta) - (1/N)\nabla_\theta f(x_n^T\theta, y_n)$ is invertible if $\mathrm{Span} \set{x_m}_{m:\, m \neq n} = \R^D$. This condition on the span holds almost surely if the $x_n$ are sampled from a continuous distribution and $D \leq N$.

\subsection{Accuracy of $\initIJ$ for regularized problems} \label{app-swissArmyProof}

As noted in the main text, \citet{giordano:2018:ourALOO} show that the error of $\initIJ$ is bounded by $C/N$ for some $C$ that is constant in $N$. However, their results apply only to the unregularized case (i.e., $\lambda = 0$). We show here that their results can be extended to the case of $\lambda > 0$ with mild additional assumptions; the proof of \cref{prop-swissArmy} appears below.
\begin{prop} \label{prop-swissArmy}
  Assume that the conditions for Corollary 1 of \citet{giordano:2018:ourALOO} are satisfied by $F(\theta)$. Furthermore, assume that we are restricted to $\theta$ in some compact subset $\Theta$ of $\R^D$, $\lambda = O(1/\sqrt{N})$, $F + \lambda R$ is twice continuously differentiable for all $\theta$, and that $\nabla^2R(\theta)$ is positive definite for all $\theta\in\Theta$. Then $\initIJ$ can be seen as an application of the approximation in Definition 2 of \citet{giordano:2018:ourALOO}. Furthermore, the assumptions of their Corollary 1 are met, which implies:
  \begin{equation} \max_{n \in [N]} \|\initTildethetan_{IJ} - \thetan\|_2 \leq  \frac{C'}{N^2} \sup_{\theta\in\Theta} \max_{n\in [N]} \n{\nabla_\theta f(x_n^T\theta, y_n)}_\infty \leq  \frac{C}{N}, \end{equation}
  where $C$ and $C'$ are problem-specific constants independent of $N$ that may depend on $D$.
\end{prop}
\cref{prop-swissArmy} provides two bounds on the error $\| \initIJ - \thetan\|_2$: either $C'/N^2$ times the maximum of the gradient or just $C/N$. One bound or the other may be easier to use, depending on the specific problem. It is worth discussing the conditions of \cref{prop-swissArmy} before going into its proof. The first major assumption is that $\theta$ is restricted to some compact set $\Theta$. Although this assumption may not be satisfied by problems of interest, one may be willing to assume that $\theta$ lives in some bounded set in practice. In any case, such an assumption seems necessary to apply the results of \citet{giordano:2018:ourALOO} to most unregularized problems, as they, for example, require $\sup_{\theta \in \Theta} F(\theta)$ to be bounded. We will require the compactness of $\Theta$ to show that $\sup_{\theta\in\Theta} F(\theta) + \lambda R(\theta)$ is bounded.
\\\\
The second major assumption of \cref{prop-swissArmy} is that $\lambda = O(1/\sqrt{N})$. We need this assumption to ensure that the term $\lambda R(\theta)$ is sufficiently well behaved. In practice this assumption may be somewhat limiting; however, we note that for fixed D, such a scaling is usually assumed -- and in some situations is necessary -- to obtain standard theoretical results for $\ell_1$ regularization (e.g., \citet{wainwright:2009:linearRegressionLasso} gives the standard scaling for linear regression, $\lambda = \Omega (\sqrt{\log(D)/N})$). Our \cref{linearRegressionTheorem,logisticRegressionTheorem} also satisfy such a scaling when D is fixed. In any case, we stress that this assumption -- as well as the assumption on compactness -- are needed only to prove \cref{prop-swissArmy}, and not any of our other results. We prove \cref{prop-swissArmy} to demonstrate the baseline results that exist in the literature so that we can then show how our results build on these baselines.
\begin{proof}
We proceed by showing that the regularized optimization problem in our \cref{ourOptProblem} can be written in the framework of Eq. (1) of \citet{giordano:2018:ourALOO} and then showing that the re-written problem satisfies the assumptions of their Corollary 1. First, the framework of \citet{giordano:2018:ourALOO} applies to weighted optimization problems of the form:
\begin{equation} \thetaw := \theta \in \Theta \; s.t. \; \frac{1}{N} \sum_{n=1}^N w_n g_n(\theta) = 0. \end{equation}
In order to match this form, we will rewrite the gradient of the objective in \cref{ourOptProblem} as a weighted sum with $N+1$ terms, where the first term, with weight $w_0 = 1$, will correspond to $R(\theta)$:
\begin{equation} \frac{1}{N+1} w_0 (N+1) \lambda \nabla R(\theta) + \frac{1}{N+1} \sum_{n=1}^N w_n \frac{N+1}{N} \nabla f(x_n^T\theta, y_n). \label{swissArmyObjective}\end{equation}
We will also need a set of weight vectors $W \subseteq \R^{N+1}$ for which we are interested in evaluating $\thetaw$. We choose this set as follows. In the set, we include each weight vector that is equal to one everywhere except $w_n=0$ for exactly one of $n \in \{1,\ldots,N\}$. Thus, for each $n$, there is a $w \in W$ such that $\thetaw = \thetan$. Finally, then, we can apply Definition 2 of \citet{giordano:2018:ourALOO} to find the approximation $\theta_{IJ}(w)$ for the $w$ that corresponds to leaving out $n$. We see that $\theta_{IJ}(w)$ in this case is exactly equal to $\initIJ$ in our notation here.

Now that we know our approximation is actually an instance of $\theta_{IJ}(w)$, we need to check that \cref{swissArmyObjective} meets the assumptions of Corollary 1 of \citet{giordano:2018:ourALOO} to apply their theoretical analysis to our problem. We check these below, first stating the assumption from \citet{giordano:2018:ourALOO} and then covering why it holds for our problem.

\begin{enumerate}
\item (\emph{Assumption 1}): \emph{for all $\theta \in \Theta$, each $g_n$ is continuously differentiable in $\theta$.}
\\ For our problem, by assumption, $R(\theta)$ and $f(x_n^T\theta, y_n)$ are twice continuously differentiable functions of $\theta$, so this assumption holds.
\item (\emph{Assumption 2}): \emph{for all $\theta \in \Theta$, the Hessian matrix, $H(\theta, 1) := (1/N) \sum_n \partial g_n(\theta) / \partial \theta^T$ is invertible and satisfies $\sup_{\theta\in\Theta} \| H\inv (\theta, 1) \|_{op} \leq C_{op} < \infty$ for some constant $C_{op}$, where $\| \cdot \|_{op}$ denotes the operator norm on matrices with respect to the $\ell_2$ norm (i.e., the maximum eigenvalue of the matrix).}
\\For our problem, by assumption, the inverse matrix $(\nabla^2 F(\theta))\inv $ exists and has bounded maximum eigenvalue for all $\theta\in\Theta$. Also by assumption, $R$ has a positive semidefinite Hessian for all $\theta$, which implies:
  $$ \sup_{\theta\in\Theta} \n{H\inv (\theta,1)}_{op} = \sup_{\theta\in\Theta} \n{\left(\nabla^2 F(\theta) + \lambda\nabla^2 R(\theta)\right)\inv }_{op} \leq \sup_{\theta\in\Theta} \n{\left(\nabla^2 F(\theta)\right)\inv}_{op}. $$
To see that the inequality holds, first note that for a positive semi-definite (PSD) matrix $A$, $\|A\inv\|_{op} = 1/ \lambda_{min}(A)$. The inequality would then follow if $\lambda_{min}(\nabla^2 F(\theta) + \lambda \nabla^2 R(\theta)) \geq \lambda_{min}(\nabla^2 F(\theta))$. To see that this holds, take any two $D \times D$ PSD matrices $A$ and $B$. Let $\lambda_d(\cdot)$ be the $d$th eigenvalue of a matrix with $\lambda_1 = \lambda_{min}$. Then:
$$ \lambda_d(A+B) = \min_{\substack{E \subseteq \R^D \\ \dim{E} = d}} \max_{\substack{x \in E \\ \n{x}_2 = 1}} x^T (A+B) x \geq \min_{\substack{E \subseteq \R^D \\ \dim{E} = d}} \max_{\substack{x \in E \\ \n{x}_2 = 1}} x^T A x = \lambda_d(A),$$
where the inequality holds because $B$ is PSD. So, $\lambda_{min}(A+B) \geq \lambda_{min}(A)$, which finishes the proof. We have thus showed that the operator norm of $H\inv (\theta,1)$ is bounded by that of $\nabla^2 F(\theta)\inv$ for all $\theta\in\Theta$.
\item (\emph{Assumption 3}): \emph{Let $g(\theta)$ and $h(\theta)$ be the $(N+1) \times D$ stack of gradients and $(N+1) \times D \times D$ stack of Hessians, respectively. That is, $g(\theta)_{nd} := (\nabla_\theta f(x_n^T\theta, y_n))_d$ for $n=1,\dots,N$ and $g(\theta)_{N+1,d} := (\nabla_\theta R(\theta))_d$, with $h$ defined similarly. Let $\n{g(\theta)}_2$ be the $\ell_2$ norm of $g$ flattened into a vector with $\n{h(\theta)}_2$ defined similarly. Then assume that there exist constants $C_g$ and $C_h$ such that:
\begin{align*}
  &\sup_{\theta\in\Theta} \frac{1}{\sqrt{N+1}} \n{g(\theta)}_2 \leq C_g < \infty \\
  &\sup_{\theta\in\Theta} \frac{1}{\sqrt{N+1}} \n{h(\theta)}_2 \leq C_h < \infty
\end{align*}
}
 To see that this holds for our problem, we have that:
 \begin{align*}
   \n{g(\theta)}_2 &:= \left[ \sum_{d=1}^D \left( (\lambda (N+1) \nabla  R(\theta)_d)^2 + \sum_{n=1}^N \left(\frac{N+1}{N}\right)^2 (\nabla f(x_n^T \theta, y_n)_d)^2 \right) \right]^{1/2} \\
   &\leq \lambda (N+1) \n{\nabla R(\theta)}_2 + \frac{N+1}{N} \left[ \sum_{d=1}^D \sum_{n=1}^N (\nabla f(x_n^T \theta, y_n)_d)^2 \right]^{1/2}.
 \end{align*}
We need to show this is bounded by $\sqrt{N+1}C_g$ for some constant $C_g$. By assumption in the statement of \cref{prop-swissArmy}, we have $\frac{1}{\sqrt{N+1}} \n{\nabla F(\theta)}_2 \leq \frac{1}{\sqrt{N}} \n{\nabla F(\theta)}_2 \leq C_g^{(1)}$ for some constant $C_g^{(1)}$. Because $\lambda$ is $O(1/\sqrt{N})$, the first term is equal to $O(\sqrt{N}) \| \nabla R(\theta)\|_2$. The compactness of $\Theta$ and the continuity of $\nabla R(\theta)$ imply that $\|\nabla R(\theta)\|_2$ is bounded by a constant for all $\theta\in\Theta$. So, we know that $\frac{O(\sqrt{N})}{\sqrt{N+1}} \| \nabla R(\theta)\|_2 \leq C_g^{(2)}$ for some constant $C_g^{(2)}$. Thus, we have that the assumption on $\| g(\theta)\|_2$ holds with $C_g = \frac{(N+1)}{N} C_g^{(1)} + C_g^{(2)}$. That the condition on $\| h(\theta) \|_2$ holds follows by the same reasoning.
\item (\emph{Assumption 4}): \emph{There exists some $\Delta_\theta > 0$ and $L_h < \infty$ such that if $\n{\theta - \hat\theta}_2 \leq \Delta_\theta$, then $\frac{1}{\sqrt{N+1}}\n{h(\theta) - h(\hat\theta)}_2 \leq L_h \n{\theta - \hat\theta}_2$.}
\\ We can show this holds for our problem by:
  \begin{align*}
    \n{h(\theta) - h(\hat\theta)}_2 &:= \\
    %& \left[\sum_{i=1}^D \sum_{j=1}^D \left(\lambda^2 (N+1)^2 (\nabla^2 R(\theta)_{ij} - \nabla^2R(\hat\theta)_{ij})^2 + \sum_{n=1}^N \frac{(N+1)^2}{N^2}( \nabla_\theta^2 f(x_n^T, y_n)_{ij} - \nabla^2_\theta f(x_n^T\hat\theta, y_n)_{ij} ) \right) \right]^{1/2}
    & \n{\frac{N+1}{N} \nabla^2 F(\theta) + \lambda (N+1) \nabla^2 R(\theta) - \frac{N+1}{N} \nabla^2 F(\hat\theta) - \lambda (N+1) \nabla^2 R(\hat\theta)}_2 \\
    & \leq (N+1)\lambda \n{\nabla^2 R(\theta) - \nabla^2 R(\hat\theta)}_2 + \frac{N+1}{N} \n{\nabla^2 F(\theta) - \nabla^2 F(\hat\theta)}_2 ,
  \end{align*}
  where we have abused notation to denote $\n{\nabla^2 F(\theta)}_2 := \sqrt{\sum_{i,j=1}^D \sum_{n=1}^N \nabla^2_\theta (f(x_n^T \theta, y_n)_{ij})^2}$. Now, we want to show that this quantity divided by $\sqrt{N+1}$ is bounded by $L_h \|\theta - \hat\theta\|_2$ for some constant $L_h$. By assumption in the statement of \cref{prop-swissArmy}, we have that Assumption 4 holds for $F$; this implies that $\frac{N+1}{(\sqrt{N+1})(N)} \n{\nabla^2 F(\theta) - \nabla^2 F(\hat\theta)}_2 \leq L_h^{(1)} \| \theta - \hat\theta\|_2$ for some constant $L_h^{(1)}$. As $R$ is twice continuously differentiable and the condition of Assumption 4 needs only to hold over a compact set of $\theta$'s, we know that $\nabla^2 R(\theta)$ is Lipschitz over this domain. Using this along with the assumption that $\lambda$ is $O(1/\sqrt{N})$, we have that:
  \begin{align*}
    \frac{\lambda (N+1)}{\sqrt{N+1}} \n{\nabla^2 R(\theta) - \nabla^2 R(\hat\theta)}_2 &= O(1) \n{\nabla^2 R(\theta) - \nabla^2 R(\hat\theta)}_2 \\
    & \leq L_h^{(2)} \n{\theta - \hat\theta}_2 ,
  \end{align*}
  for some constant $L_h^{(2)}$. So, Assumption 4 holds with constant $L_h = L_h^{(1)} + L_h^{(2)}$.
\item (\emph{Assumption 5): For all $w \in W$, we have $\frac{1}{\sqrt{N+1}} \n{w}_2 \leq C_w$ for some constant $C_w$.} This is immediately true for our definition of $W$, which, for all $w \in W$, has $\|w\|_2 = \sqrt{N}$.
\end{enumerate}

\end{proof}

\subsection{Derivation of $\initNS$} \label{app-modifiedDerivation}

\citet{wang:2018:primalDualALOO} and \citet{rad:2018:detailedALOO} derive $\initNS$ in \cref{app-modifiedApproximation} by taking a single Newton step on the objective $\objn + \lambda R$ starting at the point $\hat\theta$. For completeness, we include a derivation here. Recall that the objective with one datapoint left out is:
\begin{equation} \objn(\theta) + \lambda R(\theta) := \frac{1}{N} \sum_{m=1}^N f(x_m^T\theta, y_m) - \frac{1}{N} f(x_n^T \theta, y_n) + \lambda R(\theta), \end{equation}
which has $H(\theta) - (1/N)\nabla_\theta^2 f(x_n^T\theta, y_n)$ as its Hessian. Now consider approximating $\thetan$ by performing a single Newton step on $\objn$ starting from $\hat\theta$:
\begin{equation} \thetan \approx \hat\theta - \left(H(\hat\theta) - \frac{1}{N}\nabla_\theta^2 f(x_n^T\hat\theta, y_n)\right)\inv \left( \frac{1}{N}\sum_{m=1}^N \nabla_\theta f(x_m^T \hat\theta, y_m) - \frac{1}{N}\nabla_\theta f(x_n^T\hat\theta, y_n) + \lambda \nabla R(\hat\theta) \right). \end{equation}
Using the fact that, by definition of $\hat\theta$, $(1/N)\sum_{n=1}^N \nabla_\theta f(x_n^T \hat\theta, y_n) + \lambda \nabla R(\hat\theta) = 0$, we have that this simplifies to:
\begin{equation} \thetan \approx \hat\theta + \frac{1}{N}\left(H(\hat\theta) - \frac{1}{N}\nabla_\theta^2 f(x_n^T\hat\theta, y_n)\right)\inv  \nabla_\theta f(x_n^T\hat\theta, y_n), \end{equation}
which is exactly $\initNS$.

As $\initNS$ can be interpreted as a single Newton step on the objective $\objn + \lambda R$, it follows that $\initNS$ is exactly equal to $\thetan$ in the case that $\objn + \lambda R$ is a quadratic, as noted by \citet{beirami:2017:firstALOO}. For example, $\ell_2$ regularized linear regression has $\initNS = \thetan$ for all $n$. We further note that somewhat similar behavior can hold for $\ell_1$ regularized linear regression. Specifically, when $\mathrm{sign}\hat\theta = \mathrm{sign}\thetan$, we have that the objective $\objn + \lambda\n{\cdot}_1$ is a quadratic when restricted to the dimensions in $\hat S$. In this case, $\NS$ can be interpreted as taking a Newton step on $\objn + \lambda\n{\cdot}_1$ restricted to the dimensions in $\hat S$. It follows that $\NS = \thetan$ when $\mathrm{sign}\hat\theta = \mathrm{sign}\thetan$ for $\ell_1$ regularized linear regression.

\subsection{Computation time of approximations}

There is a major computational difference between \cref{app-modifiedApproximation} and \cref{app-regularApproximation}: the former requires the inversion of a $D \times D$ matrix for \emph{each} $\thetan$ approximated, while the latter requires a single $D \times D$ matrix inversion for $\emph{all}$ $\thetan$ inverted, which incurs a cost of $O(ND^3)$ versus a cost of $O(D^3)$. Even for small $D$, this is a significant additional expense. 

However, as noted by \citet{rad:2018:detailedALOO, wang:2018:primalDualALOO}, \cref{app-modifiedApproximation} is much cheaper when considering the special case of generalized linear models. In this case, $\nabla_\theta^2 f_n$ is some scalar times $x_n x_n^T$ -- a rank one matrix. The Sherman-Morrison formula then allows us to cheaply compute the needed inverse in \cref{app-modifiedApproximation} given only $H\inv$; this is how Equation 8 in \citet{rad:2018:detailedALOO} and Equation 21 in \citet{wang:2018:primalDualALOO} are derived. Even though we only consider GLMs in this work, we still study \cref{app-regularApproximation} with the hope of retaining scalability in more general problems.

\section{Derivation of $\IJ$ and $\NS$ via smoothed approximations} \label{app-smoothing}

As noted in \cref{sec-approximation}, \citet{rad:2018:detailedALOO,wang:2018:primalDualALOO} derive the $\NS$ approximation by considering $\initNS[R^\eta]$ with $R^\eta$ being some smoothed approximation to the $\ell_1$ norm, and then taking the limit of $\initNS[R^\eta]$ as the amount of smoothness goes to zero. We review this approach and then state our \cref{prop-restrictedApproximation}, which says that the same technique can be used to derive $\IJ$.

We first give two possible ways to smooth the $\ell_1$ norm. The first is given by \citet{rad:2018:detailedALOO}:
\begin{equation}
	\n{\theta}_1 \approx R^\eta(\theta) := \sum_{d=1}^D \frac{1}{\eta}\bigg( \log(1 + e^{\eta\theta_d}) + \log(1 + e^{-\eta\theta_d}) \bigg),
	\label{radSmoother}
\end{equation}
The second option is to use the more general smoothing framework described by \citet{wang:2018:primalDualALOO}. They allow selection of a function $q: \R \to \R$ satisfying: (1) $q$ has compact support, (2) $\int q(u) \; du = 1$, $q(0) > 0$, and $q \geq 0$, and (3) $q$ is symmetric around 0 and twice continuously differentiable on its domain, and then define a smoothed approximation:
\begin{equation}
	R^\eta (\theta) := \eta \sum_{d=1}^D \int_{-\infty}^\infty \abs{u} q\big( \eta (\theta_d - u)\big) du,  	\label{wangSmoother}
\end{equation}
In both \cref{radSmoother,wangSmoother}, we have $\lim_{\eta\to\infty} = \|\theta\|_1$. Notice that either choice of $R^\eta$ is twice differentiable for any $\eta < \infty$, so one can consider the approximations $\initNS[R^\eta], \initIJ[R^\eta]$. We now state two assumptions, both of which are given by \citet{rad:2018:detailedALOO,wang:2018:primalDualALOO}, under which one can show the limits of these approximations as $\eta\to\infty$ are equal to $\NS$ and $\IJ$.

\begin{assumption} \label{assum:subgradient}
  For any element $\hat z \in \R^D$ of the subdifferential $\partial \n{\theta}_1$ evaluated at $\hat\theta$ such that $\nabla F(\hat\theta) + \lambda \hat z = 0$, we have $\n{\hat z_{\hat S^c}}_\infty < 1$.
\end{assumption}
\begin{assumption} \label{assum:continuouslyDifferentiable}
  For any $y_n \in \R$, $f(z,y_n)$ is a twice continuously differentiable function as a function of $z \in \R$.
\end{assumption}
\begin{prop}[Theorem 1 of \citet{rad:2018:detailedALOO}; Theorem 4.2 of \citet{wang:2018:primalDualALOO}] \label{prop-restrictedApproximation-modified} Take \cref{assum:subgradient,assum:continuouslyDifferentiable}. Suppose $H_{\hat S \hat S}$ has strictly positive eigenvalues. Let $H^{\bn}_{\hat S \hat S} := H_{\hat S \hat S} - [\nabla^2_\theta f(x_n^T\hat\theta, y_n)]_{\hat S \hat S}$, and suppose that, for all $n$, $H^{\bn}_{\hat S \hat S}$ is invertible.
Then, for $R^\eta$ as in \cref{radSmoother} or \cref{wangSmoother},
\begin{equation}
\NS := \lim_{\eta\to\infty} \initNS[R^\eta] =
		\begin{pmatrix}
			\hat\theta_{\hat S} + (H^{\bn}_{\hat S \hat S})\inv \left[ \nabla_\theta f(x_n^T \hat\theta, y_n)\right]_{\hat S} \\ 			
		\end{pmatrix}. 
\end{equation}
\end{prop}
As noted in the main text, we show that a very similar result holds for the limit of $\initIJ[R^\eta]$:
\begin{prop} \label{prop-restrictedApproximation}
Take \cref{assum:subgradient,assum:continuouslyDifferentiable}. Suppose $H_{\hat S \hat S}$ is invertible. Then for $R^\eta$ as in \cref{radSmoother} or \cref{wangSmoother}:
\begin{equation}
\IJ := \lim_{\eta\to\infty} \initIJ[R^\eta] =		
         \begin{pmatrix}
			\hat\theta_{\hat S} + H_{\hat S \hat S}\inv \left[ \nabla_\theta f(x_n^T \hat\theta, y_n)\right]_{\hat S} \\
			0
		\end{pmatrix}.
\end{equation}
\end{prop}

The proof of \cref{prop-restrictedApproximation} is a straightforward adaptation of the proof of \cref{prop-restrictedApproximation-modified}. We prove it separately for the two different forms of $R^\eta$ in the next two subsections.

\subsection{Proof of \cref{prop-restrictedApproximation} using \cref{radSmoother}}

This proof is almost identical to the proof of Theorem 1 from \citet{rad:2018:detailedALOO}. First we will need some notation. Let $\hat\theta^\eta$ be the solution to \cref{ourOptProblem} using $R^\eta$ from \cref{radSmoother} as the regularizer. Let $\hat S_\eta := \setc{i}{\abs{\hat\theta^\eta} > c / \eta}$ for some constant $c$. We know from the arguments in Appendix A.2 of \citet{rad:2018:detailedALOO} that for an appropriately chosen $c$ and $\eta > C$ for some large constant $C > 0$, we have $S^\eta = \hat S =: \supp\hat\theta$. Next, define the scalars $\dnoneeta$ and $\dntwoeta$ as the derivatives of $f$ evaluated at $\hat\theta^\eta$:
\begin{equation} \dnoneeta := \frac{d f(z,y_n)}{dz}\at{z = x_n^T \hat\theta^\eta} \; , \quad \dntwoeta := \frac{d^2 f(z,y_n)}{dz^2}\at{z = x_n^T \hat\theta^\eta}. \end{equation}
Finally, divide the Hessian of the smoothed problem up into blocks by defining:
\begin{align*}
  &A := X_{\cdot, \hat S_\eta}^T \mathrm{diag}\set{\dntwoeta} X_{\cdot, \hat S_\eta} + \lambda \nabla^2 R^\eta (\hat\theta^\eta), \quad\quad B:= X_{\cdot, \hat S_\eta^c}^T \mathrm{diag}\set{\dntwoeta} X_{\cdot, \hat S_\eta} + \lambda \nabla^2 R^\eta (\hat\theta^\eta) \\
  & C := X_{\cdot, \hat S_\eta^c}^T \mathrm{diag}\set{\dntwoeta} X_{\cdot, \hat S_\eta^c} + \lambda \nabla^2 R^\eta (\hat\theta^\eta), \quad\quad D := (A - BC\inv B^T)\inv
\end{align*}
We can then compute the block inverse of the Hessian of the smoothed problem, $H_\eta\inv$ as:
\begin{equation} H_\eta\inv = \begin{pmatrix} A & B \\ B^T & C \end{pmatrix}\inv = \begin{pmatrix} D & -DBC\inv \\ -C\inv B^T D & A\inv + A\inv BDB^T A\inv \end{pmatrix}.\end{equation}
\citet{rad:2018:detailedALOO} show that all blocks of $H_\eta\inv$ converge to zero as $\eta \to \infty$ except for the upper left, which has $D \to X_{\cdot, \hat S}^T \mathrm{diag}\set{\dntwo}X_{\cdot,\hat S}$. So, we have that the limit of $\initIJ[R^\eta]$ is:
\begin{equation} \lim_{\eta\to\infty} \initIJ[R^\eta] = \lim_{\eta\to\infty} H_\eta\inv \dnoneeta x_n = \dnone \begin{pmatrix} (X_{\cdot, \hat S}^T \mathrm{diag}\set{\dntwo} X_{\cdot,\hat S})\inv & 0 \\ 0 & 0 \end{pmatrix} \begin{pmatrix} x_{n \hat S} \\ x_{n \hat S^c} \end{pmatrix}, \label{limThetaEta} \end{equation}
where we used that $\hat\theta_\eta \to \hat\theta$ by Lemma 15 of \citet{rad:2018:detailedALOO}, which gives that $\dnoneeta \to \dnone$ by \cref{assum:continuouslyDifferentiable}. The resulting approximation is exactly that given in the statement of \cref{prop-restrictedApproximation} by noting that $\dnone x_{n\hat S} = [\nabla_\theta f(x_n^T \hat\theta, y_n)]_{\hat S}$.

\subsection{\cref{prop-restrictedApproximation} using \cref{wangSmoother}}

This proof proceeds along the exact same direction as when using \cref{radSmoother}. In their proof of their Theorem 4.2, \citet{wang:2018:primalDualALOO} provide essentially all the same ingredients that \citet{rad:2018:detailedALOO} do, except for the general class of smoothed approximations given by \cref{wangSmoother}. This allows the same argument of taking the limit of each block of the Hessian individually and finishing by taking the limit as in \cref{limThetaEta}.

\section{The importance of correct support recovery} \label{app-supportRecoveryExperiments}
\cref{mainResult} shows that each $\thetan$ having correct support (i.e., $\supp\thetan = \supp\theta^*$) is a sufficient condition for obtaining the fixed-dimensional error scaling shown in blue in \cref{fig-errorScalingExample}. Here, we give some brief empirical evidence that this condition is necessary in the case of linear regression when using $\IJ$ as an approximation. For values of $N$ ranging from 1,000 to 8,000, we set $D = N/10$ and generate a design matrix with i.i.d.\ $N(0,1)$ entries. The true $\theta^*$ is supported on its first five entries, with the rest set to zero. We then generate observations $y_n = x_n^T \theta^* + \eps_n$, for $\eps_n \overset{i.i.d.}{\sim} N(0,1)$.

To examine what happens when the recovered supports are and are not correct, we use slightly different values of the regularization parameter $\lambda$. Specifically, the results of \citet{wainwright:2009:linearRegressionLasso} (especially their Theorem 1) tell us that the support recovery of $\ell_1$ regularized linear regression will change sharply around $ \lambda \approx 4 \sqrt{\log(D) / N},$ where lower values of $\lambda$ will fail to correctly recover the support. With this in mind, we choose two settings of $\lambda$: $1.0 \sqrt{\log (D) / N}$ and $10.0 \sqrt{\log(D) / N}$. As expected, the righthand side of \cref{supportScalingExperiment} shows that the accuracy of $\IJ$ is drastically different in these two situations. The lefthand plot of \cref{supportScalingExperiment} offers an explanation for this observation: the support of $\supp\thetan$ grows with $N$ under the lower value of $\lambda$, whereas the larger value of $\lambda$ ensures that $\abs{\supp\thetan} = \abs{\supp\theta^*} = \mathrm{const}$. Empirically, these results suggest that, for high-dimensional problems, approximate CV methods are accurate estimates of exact CV only when taking advantage of some kind of low ``effective dimensional'' structure.

\begin{figure*}[!ht]
  \centering
  \begin{tabular}{cc}
    \includegraphics[scale=.4, trim={1.0cm, 0, 0, 0}]{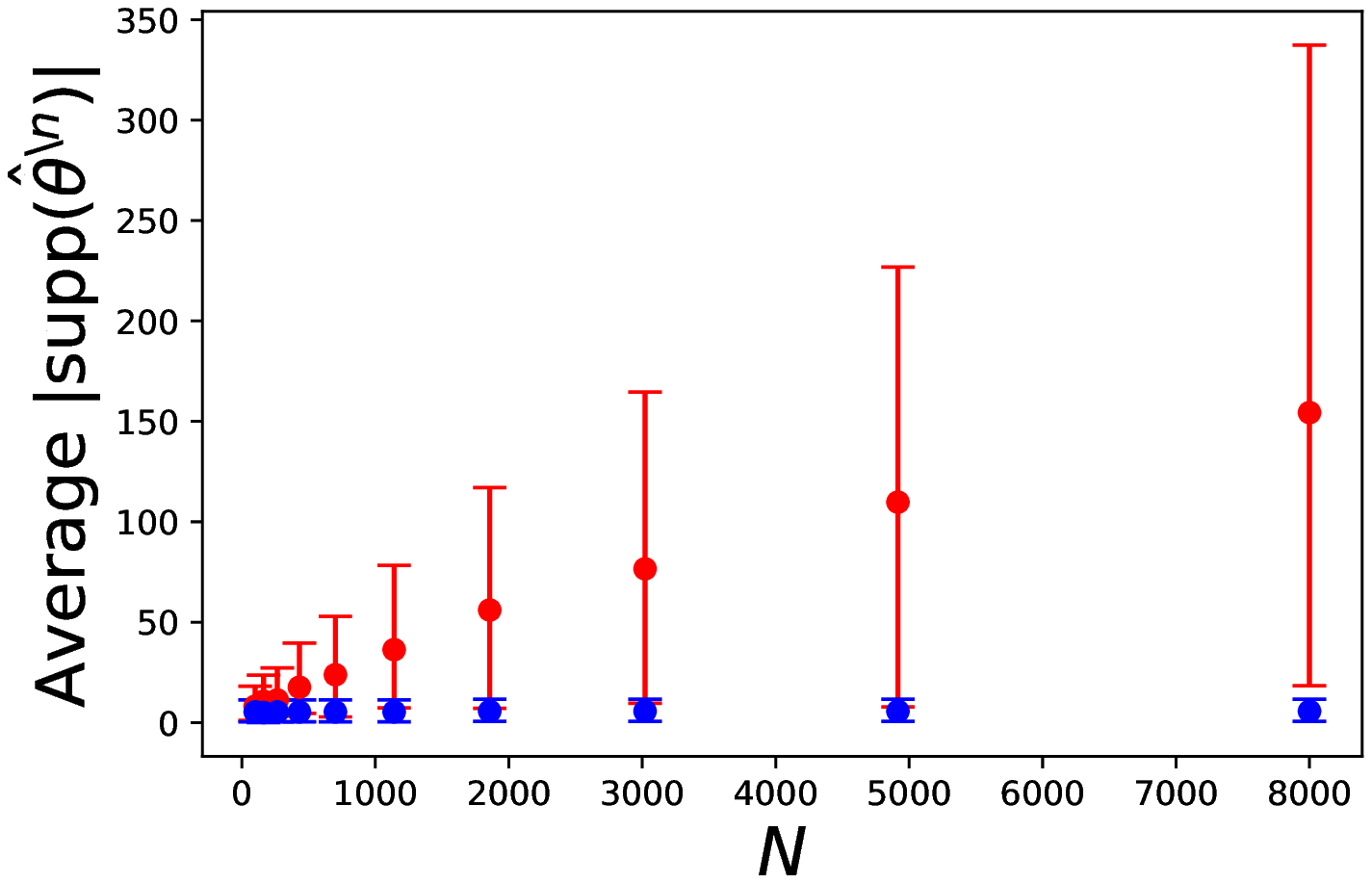} &
    \includegraphics[scale=.4, trim={1.0cm, 0, 0, 0}]{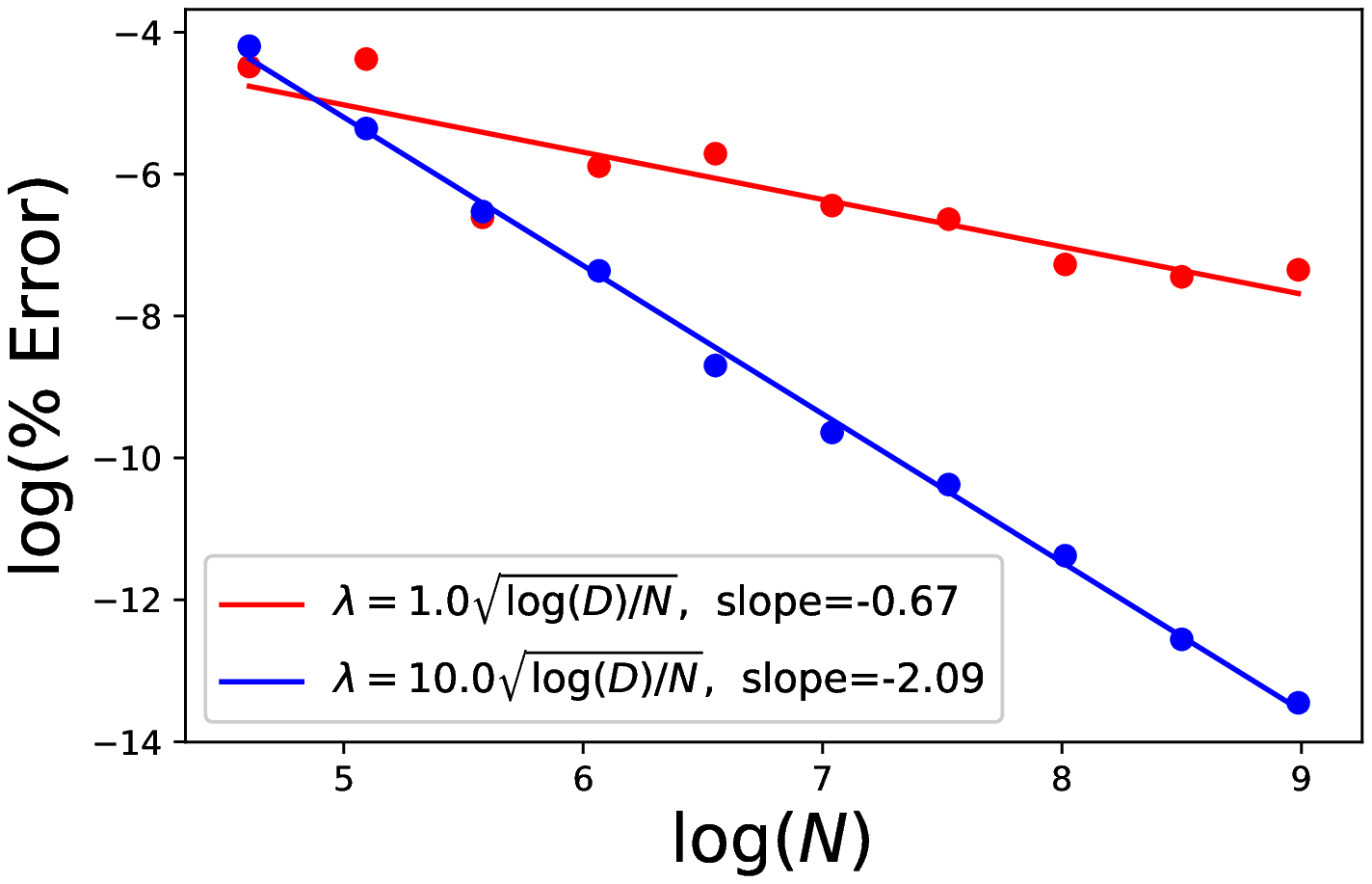}
  \end{tabular}
  \caption{Illustration of the role of support recovery in the accuracy of $\IJ$ in the case of linear regression. \emph{Left}: Points show the average of $\abs{\supp\thetan}$ over random values of $n$. Error bars show the min and max $\abs{\supp\thetan}$ over these $n$. For $\lambda = 10.0 \sqrt{\log (D) / N}$ (blue), the mean recovered support is constant with $N$. For $\lambda = 1.0 \sqrt{\log (D) / N}$ (red), $\abs{\supp\thetan}$ grows with $N$, and varies dramatically for different values of $n$. \emph{Right}: Percent error (\cref{percentError}) as $D$ scales with $N$. When the support recovery is constant, we recover an error scaling of roughly $1/N^2$, whereas a growing support results in a much slower decay.}
\label{supportScalingExperiment}
\end{figure*}

That the approximation quality relies so heavily on the exact setting of $\lambda$ is somewhat concerning. However, we emphasize that sensitivity exists for $\ell_1$ regularization in general; as previously noted, \citet{wainwright:2009:linearRegressionLasso} demonstrated similarly drastic behavior of $\supp\hat\theta$ in the same exact linear regression setup that we use here. On the other hand, \citet{homrighausen:2014:l1LOO} do show that using exact LOOCV to select $\lambda$ for $\ell_1$ regularized linear regression gives reasonable results. In \cref{app-lambdaSelection}, we empirically show this is sometimes, but not always, the case for our and other approximate CV methods.

\paragraph{Accuracy of approximate CV by optimization error.} In early experiments, we used the Python bindings for the glmnet package \citep{friedman:2009:glmnetPaper} to solve our $\ell_1$ regularized problems. However, we found that both $\IJ$ and $\NS$ failed to recover the roughly $1/N^2$ scaling present in fixed-dimensional problems (e.g. as shown in \cref{fig-errorScalingExample} of \cref{introduction}) that we would expect given our theoretical results. We found that this was due to the relatively loose convergence tolerance with which glmnet is implemented (e.g. parameter changes of $\leq 1\times 10^{-4}$ between iterations), which seems to be an issue for approximate CV methods and related approximations \citep{giordano:2018:ourALOO,giordano:2015:LRVB}. We implemented our own $\ell_1$ solver in Python using many of the speed-ups proposed in \citet{friedman:2009:glmnetPaper} and set a convergence theshold of $1\times 10^{-10}$ for the initial fit of $\hat\theta$. This solver was used to produce all of our results, including \cref{supportScalingExperiment}, which shows the expected roughly $1/N^2$ accuracy of $\IJ$ in blue.

\section{Details of real experiments} \label{app-realExperiments}

We use three publicly available datasets for our real-data experiments in \cref{sec-experiments}:
\begin{enumerate}
\item The ``Gisette'' dataset \cite{gisette} is available from the UCI repository at \url{https://archive.ics.uci.edu/ml/datasets/Gisette}. The dataset is constructed from the MNIST handwritten digits dataset. Specifically, the task is to differentiate between handwritten images of either ``4'' or ``9.'' There are $N=6,000$ training examples, each of which has $D=5,000$ features, some of which are junk ``distractor features'' added to make the problem more difficult.
\item The ``bcTCGA'' \cite{bcTCGA} is a dataset of breast cancer samples from The Cancer Genome Atlas, which we downloaded from \url{http://myweb.uiowa.edu/pbreheny/data/bcTCGA.html}. The dataset consists of $N=536$ samples of tumors, each of which has the real-valued expression levels of $D = 17,322$ genes. The task is to predict the real-valued expression level of the BRCA1 gene, which is known to correlate with breast cancer.
\item The ``RCV1'' dataset \cite{rcv1} is a dataset of Reuters' news articles given one of four categorical labels according to their subject: ``Corporate/Industrial,'' ``Economics,'' ``Government/Social,'' and ``Markets.'' We use a pre-processed binarized version from \url{https://www.csie.ntu.edu.tw/~cjlin/libsvmtools/datasets/binary.html}, which combines the first two categories into a ``positive'' label and the latter two into a ``negative'' label. The full dataset contains $N=20,242$ articles, each of which has $D=47,236$ features. Running exact CV on this dataset would have been prohibitively slow, so we created a smaller dataset. First, the covariate matrix $X$ is extremely sparse (i.e., most entries are zero), so we selected the top 10,000 most common features and threw away the rest. We then randomly chose 5,000 documents to keep as our training set. After throwing away any of the 10,000 features that were now not observed in this subset, we were left with a dataset of size $N = 5,000$ and $D = 9,836$.
  %\item The ``p53 mutants'' dataset \cite{p53_1, p53_2, p53_3}, available from the UCI repository \url{https://archive.ics.uci.edu/ml/datasets/p53+Mutants}. The dataset contains $N = 16,772$ observations of p53 proteins, each of which has been given a binary label of ``active'' or ``inactive.'' Each observation comes with $D = 5,409$ observed features of the protein
\end{enumerate}

In order to run $\ell_1$ regularized regression on each of these datasets, we first needed to select a value of $\lambda$. Since all of these datasets are fairly high dimensional, our experiments in \cref{app-lambdaSelection} suggests our approximation will be inaccurate for values of $\lambda$ that are ``too small.'' In an attempt to get the order of magnitude for $\lambda$ correct, we used the theoretically motivated value of $\lambda = C \sqrt{\log(D) / N}$ for some constant $C$ (e.g., \citet{li:2015:sparsistency} shows this scaling of $\lambda$ will recover the correct support for both linear and logistic regression). \cref{sec-experiments} suggests that the constant $C$ can be very important for the accuracy of our approximation, and our experiments there suggest that inaccuracy is caused by too large a recovered support size $\abs{\supp\hat\theta}$. For the RCV1 and Gisette datasets, both run with logistic regression, we guessed a value of $C=1.5$, as this sits roughly in the range of values that give support recovery for logistic regression on synthetic datasets. After confirming that $\abs{\supp\hat\theta}$ was not too large (i.e., of size ten or twenty), we proceeded with these experiments. Although we found linear regression on synthetic data typically needed a larger value of $C$ than logistic regression on synthetic data, we found that $C = 1.5$ also produced reasonable results for the bcTCGA dataset.

%For these experiments, we guessed a reasonable sounding value for $C$, solved for $\hat\theta$, and confirmed that $\abs{\supp\hat\theta}$ was not too large. For all of the datasets here, we used $C = 1.50$ to get the results reported in \cref{sec-realExperiments}.

\section{Selection of $\lambda$} \label{app-lambdaSelection}

Our work in this paper is almost exclusively focused on approximating CV for model assessment. However, this is not the only use-case of CV. CV is also commonly used for model selection, which, as a special case, contains hyperparameter tuning. Previous authors have used approximate CV methods for hyperparameter tuning in the way one might expect: for various values of $\lambda$, compute $\hat\theta$ and then use approximate CV to compute the out-of-sample error of each $\hat\theta$; the $\lambda$ leading to the lowest out-of-sample error is then selected \citep{obuchi:2016:linearALOO,obuchi:2018:logisticALOO,beirami:2017:firstALOO,rad:2018:detailedALOO,wang:2018:primalDualALOO,giordano:2018:ourALOO}. While many of these authors theoretically study the accuracy of approximate CV, we note that they only do so in the context of model \emph{assessment} and only empirically study approximate CV for hyperparameter tuning. In this appendix, we add to these experiments by showing that approximate CV can exhibit previously undemonstrated complex behavior when used for hyperparameter tuning.

We generate two synthetic $\ell_1$ regularized logistic regression problems with $N = 300$ observations and $D = \{75, 150\}$ dimensions. The matrix of covariates $X$ has i.i.d.\ $N(0,1)$ entries, and the true $\theta^*$ has its first five entries drawn i.i.d.\ as $N(0,1)$ with the rest set to zero. As a measure of the true out of sample error, we construct a test set with ten thousand observations. For a range of values of $\lambda$, we find $\hat\theta$, and measure the train, test, exact LOOCV, and approximate LOOCV errors via both $\NS$ and $\IJ$; the results are plotted in \cref{lambdaSelectionResults}. $\NS$ (blue dashed curve) is an extremely close approximation to exact CV (red curve) in both datasets and selects a $\lambda$ that gives a test error very close to the $\lambda$ selected by exact CV. On the other hand, $\IJ$ (solid blue curve) performs very differently on the two datasets. For $D = 75$, it selects a somewhat reasonable value for $\lambda$; however, for $D=150$, $\IJ$ goes disastrously wrong by selecting the obviously incorrect value of $\lambda = 0$. While the results in \cref{lambdaSelectionResults} come from using our $\IJ$ to approximate CV for an $\ell_1$ regularized problem, we note that this issue is not specific to the current work; we observed similar behavior when using $\ell_2$ regularization and the pre-existing $\initIJ[\ell_2]$.

While $\NS$ performs far better than $\IJ$ in the experiments here, it too has a limitation when $D > N$. In particular, when $\lambda$ is small enough, we will eventually recover $\abs{\hat S} = N$. At this point, the matrix we need to invert in the definition of $\NS$ in \cref{restrictedApproximations} will be a $N \times N$ matrix that is the sum of $N-1$ rank-one matrices. As such, it will not be invertible, meaning that we cannot compute $\NS$ for small $\lambda$ when $D > N$. Even when $D$ is less than -- but still close to -- $N$, we have observed numerical issues in computing $\NS$ when $\lambda$ is sufficiently small; typically, these issues show up as enormously large values for ALOO for small values of $\lambda$.

Given the above discussion, we believe that an understanding of the behavior of $\IJ$ and $\NS$ for the purposes of hyperparameter tuning is a very important direction for future work.

%Still, all is not lost for approximate CV: the righthand side \cref{lambdaSelectionResults} shows that for the same problem setup with $D=75$ dimensions, the error vs $\lambda$ curve constructed by approximate CV is significantly different. In particular, it is convex and has its minimum very close to that of exact CV. We believe a further understanding of this issue is the most pressing direction for future work. In the meantime, this failure mode of approximate CV is at least easy to spot, assuming one believes the true out of sample error does not really have a minimum at $\lambda = 0$.

\begin{figure*}[h]
  \centering
  \includegraphics[scale=.5]{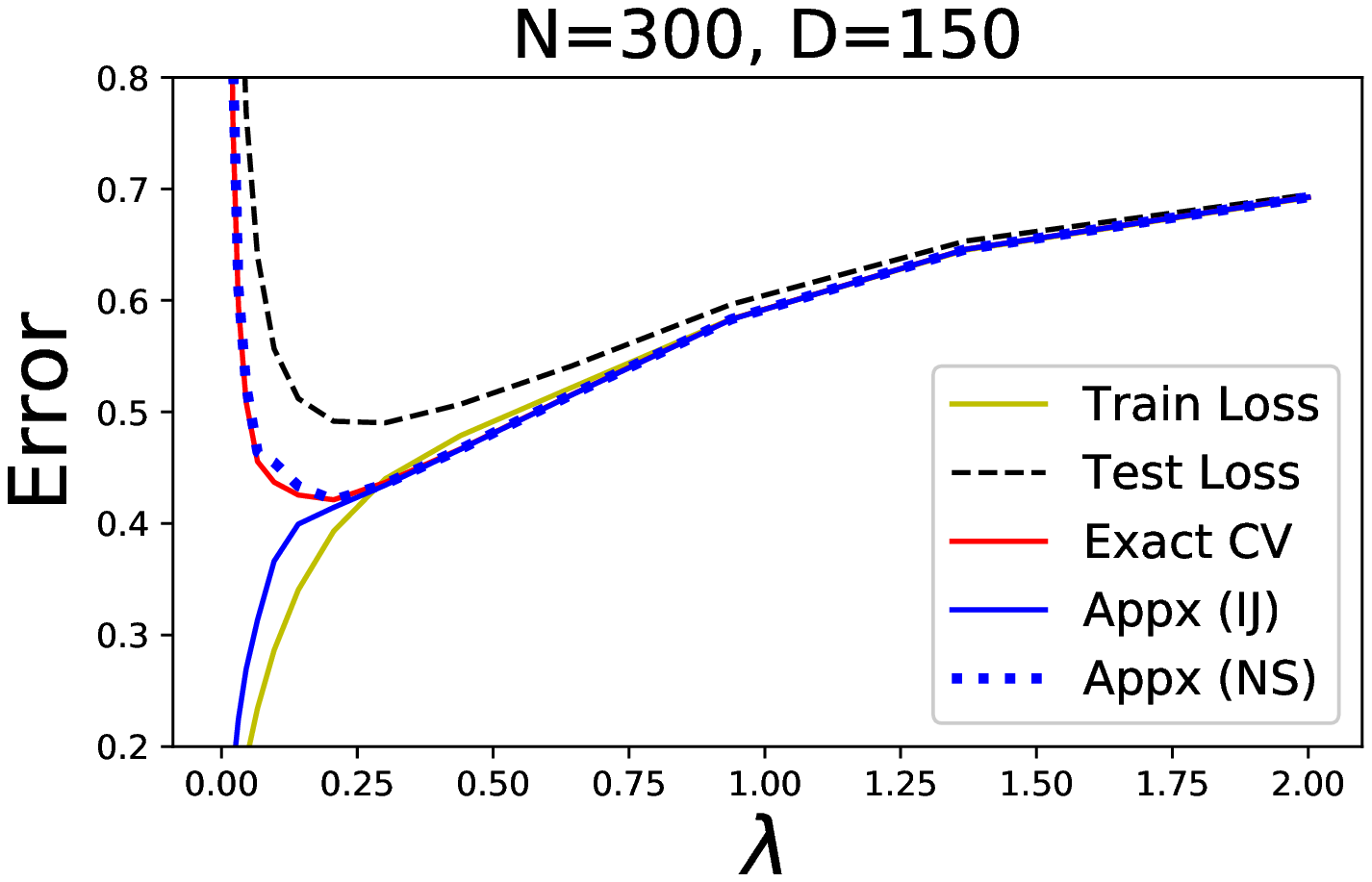}
  \includegraphics[scale=.5]{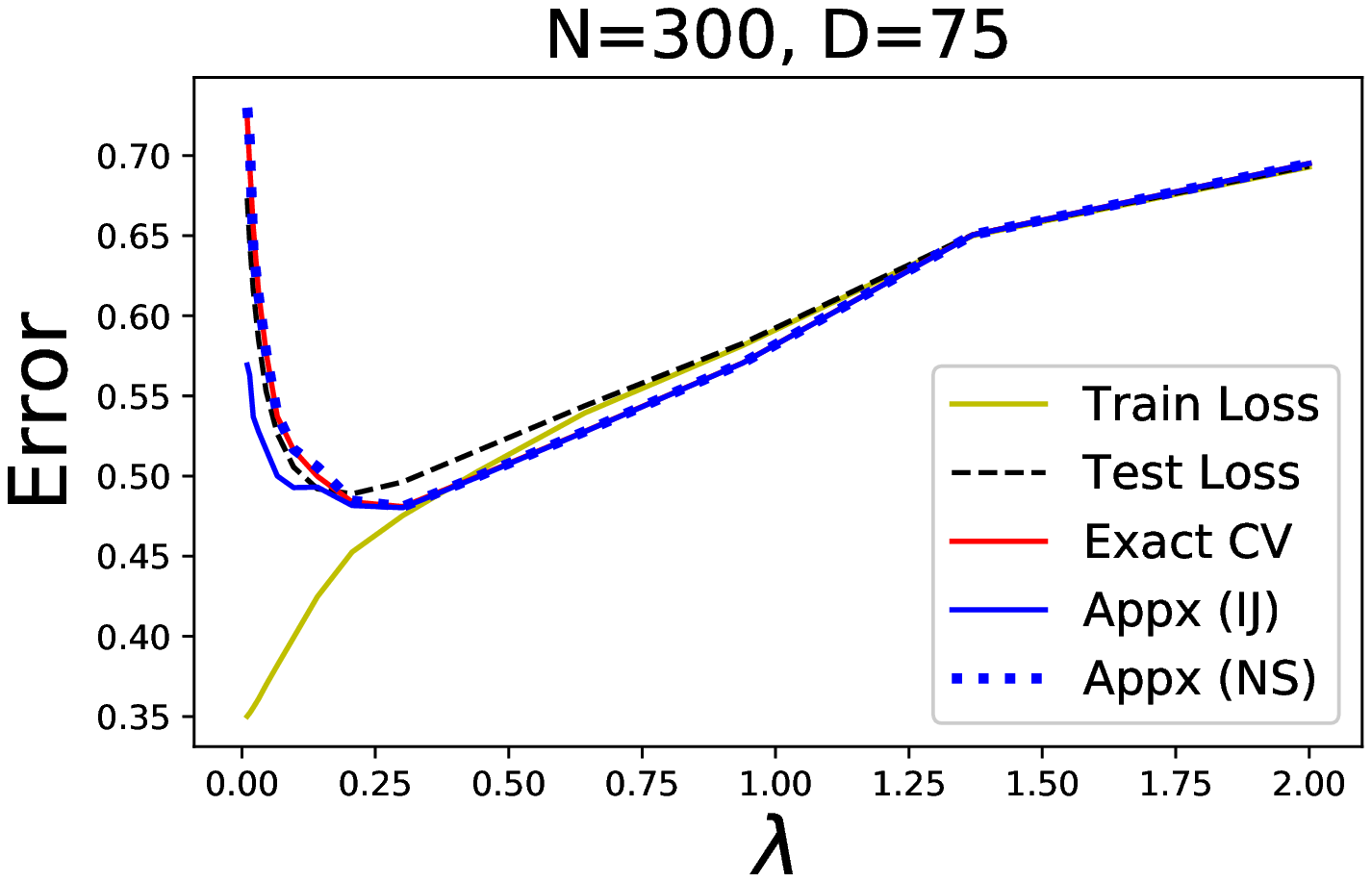}
  \caption{Experiment for selecting $\lambda$ from \cref{app-lambdaSelection}. (\emph{Top:}) Despite being very accurate for higher values of $\lambda$, the degredation of the accuracy of $\IJ$ for lower values of $\lambda$ (which corresponds to a larger $\hat S$) causes the selection of a $\lambda$ that is far from optimal in terms of test loss. (\emph{Bottom:}) For a lower dimensional problem, the curve constructed by $\IJ$ much more closely mirrors that of exact CV for all values of $\lambda$. In both cases, $\NS$ performs well.} \label{lambdaSelectionResults}
\end{figure*}

After the initial posting of this work, \citet{wilson:2020:modelSelectionALOO} provided a more thorough investigation of model selection using $\NS$ and $\IJ$. Their work gives an analytical example (as opposed to our empirical example here) showing that $\IJ$ can fail for model selection in sparse models. They further propose a modification to $\IJ$ based on proximal operators that avoids this issue in both theory and practice.

\section{Proofs from \cref{sec-theory}} \label{app-proofs}

As mentioned in the main text, there exist somewhat general assumptions in the $\ell_1$ literature under which $\supp\hat\theta = S$ \citep{lee:2014:generalSparseRecovery,li:2015:sparsistency}. By taking these assumptions for all leave-one-out problems, we immediately get that $\supp\thetan = S$ for all $n$. Our method for proving \cref{linearRegressionTheorem,logisticRegressionTheorem} will be to show that the assumptions of those theorems imply those from the $\ell_1$ literature for all leave-one-out problems.

\subsection{Assumptions from \citet{li:2015:sparsistency}} \label{app-liAssumptions}

We choose to use the conditions from \citet{li:2015:sparsistency}, as we find them easier to work with for our problem. \citet{li:2015:sparsistency} gives conditions on $F = (1/N) \sum_n f(x_n^T\theta, y_n)$ under which $\supp\hat\theta = S$. We are interested in $\supp\thetan$, so we state versions of these conditions for $\objn := (1/N) \sum_{m:\, m\neq n} f(x_m^T\theta, y_n)$.

\begin{assumption}[LSSC] \label{assum:LSSC}
  $\forall n$, $\objn$ satisfies the $(\theta^*, \R^D)$ locally structured smoothness condition (LSSC)\footnote{Readers familiar with the LSSC may see choosing the neighborhood of $\theta^*$ as $\R^D$ to be too restrictive. This choice is not necessary for our results; we state \cref{assum:LSSC} this way only for simplicity. See \cref{app-LSSC} for an explanation.} with constant $K$. We recall this condition, due to \citet{li:2015:sparsistency}, in \cref{app-LSSC}.
  %Each $\objn$ satisfies the $(\theta^*, N_{\theta^*})$ locally structured smoothness condition with constant $K$. This condition, proposed by \cite{li:2015:sparsistency}, is not required to understand our results, so we defer its definition to \cref{app-LSSC}.
\end{assumption}
\begin{assumption}[Strong convexity] \label{assum:lambdaMin}
For a matrix $A$, let $\lambda_{min}(A)$ be the smallest eigenvalue of $A$. Then, $\forall n$ and for some constant $L_{min}$, the Hessian of $\objn$ is positive definite at $\theta^*$ when restricted to the dimensions in $S$:
$
	\lambda_{min}\left( \nabla^2 \objn(\theta^*)_{SS}\right) \geq L_{min} > 0.
$
\end{assumption}
\begin{assumption}[Incoherence] \label{assum:incoherence}
$\forall n$ and for some $\gamma > 0$,
  \begin{equation} \n{\nabla^2 \objn(\theta^*)_{S^c, S} \left( \nabla^2 \objn(\theta^*)_{SS}\right)\inv}_{\infty} < 1 - \gamma. \label{defIncoherence} \end{equation}
\end{assumption}
%
%In the case of linear regression, incoherence implies a lack of correlation between columns of $X$ in $S$ and those in $S^c$. Specifically, in this case, \cref{defIncoherence} reads $\|(X_{\cdot,S}^T X_{\cdot,S})\inv X_{\cdot,S}^T X_{\cdot,S^c}\|_\infty < 1-\gamma$. That is, the assumption constrains the regression coefficients of the columns of $X_{\cdot, S^c}$ on the columns of $X_{\cdot, S}$.
%
\begin{assumption}[Bounded gradient] \label{assum:boundedGradient}
For $\gamma$ from \cref{assum:incoherence}, $\forall n$, the gradient of $\objn$ evaluated at the true parameters $\theta^*$ is small relative to the amount of regularization:
$
	\n{\nabla \objn(\theta^*)}_\infty  \leq (\gamma / 4) \lambda.
$
\end{assumption}
\begin{assumption}[$\lambda$ sufficiently small] \label{assum:lambdaSmall}
  For $K,L_{min}$ and $\gamma$ as in \cref{assum:LSSC,assum:lambdaMin,assum:incoherence}, the regularization parameter is sufficiently small:
$
	\lambda <  L^2_{min} \gamma / ( 4(\gamma + 4)^2 \deff K ), %\label{lambdaUpperBound}
$
 	where there is no constraint on $\lambda$ if $K = 0$.
\end{assumption}
We see in \cref{app-supportInclusion} that a minor adaptation of Theorem 5.1 from \citet{li:2015:sparsistency} tells us that \cref{assum:LSSC,assum:lambdaMin,assum:incoherence,assum:boundedGradient,assum:lambdaSmall} imply $\forall n, \supp\thetan \subseteq S$. To prove the accuracy of $\NS$ and $\IJ$, though, we further need that $\supp\thetan \subseteq \hat S$ so that all LOOCV problems run over the same low-dimensional space as the full-data problem. It will be easier to state conditions for a stronger result, that $\supp\thetan = \hat S = S$. This will follow from an assumption on the smallest entry of $\theta^*_S$, which we stated as \cref{assum:thetaMin} in the main text. We stated \cref{assum:thetaMin} using the quantity $T_{min}$ to avoid stating \cref{assum:lambdaMin,assum:incoherence} in the main text. We can now state its full version.
\begin{assumption}[full version of \cref{assum:thetaMin}]  \label{assum:thetaMin-app}
For $L_{min}$ and $\gamma$ from \cref{assum:lambdaMin,assum:incoherence}, 
$
	\min_{s \in S} \abs{\theta^*_s} > ( \sqrt{\deff}(\gamma + 4) / L_{min} ) \lambda.
  %\label{betaMinCondition} \end{equation}
$
\end{assumption}
\begin{prop} \label{prop-supportRecovery}
  If \cref{assum:LSSC,assum:lambdaMin,assum:incoherence,assum:boundedGradient,assum:lambdaSmall,assum:thetaMin-app} hold, then $\forall n, \supp\thetan = \hat S = S$.
\end{prop}
\begin{proof}
  This is immediate from Theorem 5.1 of \citet{li:2015:sparsistency}.
  %Under \cref{assum:LSSC,assum:lambdaMin,assum:incoherence,assum:boundedGradient,assum:lambdaSmall,assum:thetaMin}, Theorem 5.1 of \citet{li:2015:sparsistency} implies that $\forall n, \supp\thetan = S$ and that $\hat S = S$.
\end{proof}

\subsection{Local structured smoothness condition (LSSC)} \label{app-LSSC}

We now define the local structured smoothness condition (LSSC). The LSSC was introduced by \citet{li:2015:sparsistency} for the purpose of extending proof techniques for the support recovery of $\ell_1$ regularized linear regression to more general $\ell_1$ regularized $M$-estimators. Essentially, it provides a condition on the smoothness of the third derivatives of the objective $F(\theta)$ near the true sparse $\theta^*$. One can then analyze a second order Taylor expansion of the loss and use the LSSC to show that the remainder in this expansion is not too large. To formalize the LSSC, we need to define the third order derivative of $F$ evaluated along a direction $u \in \R^D$:
$$ D^3 F(\theta)[u] := \lim_{t \to 0} \frac{\nabla^2 F(\theta + tu) - \nabla^2 F(\theta)}{t}.$$
In the cases considered in this paper, this is just a $D \times D$ matrix. We can then naturally define the scalar $D^3F(\theta) [u,v,w]$ as an outer product on this matrix:
$$ D^3[u,v,w] :=  v^T \big( D^3F(\theta)[u] \big) w $$
\begin{defn}[LSSC] Let $F: \R^D \to \R$ be a continuously three-times differentiable function. For $\theta^* \in \R^D$ and $N_{\theta^*} \subseteq \R^D$, the function $F$ satisfies the $(\theta^*, N_{\theta^*})$ LSSC with constant $K \geq 0$ if for any $u \in \R^D$:
  \begin{equation} \abs{D^3 f(\theta^* + \delta)[u,u,e_j]} \leq K\n{u}_2^2, \end{equation}
  where $e_j \in \R^D$ is the $j$th coordinate vector, and $\delta \in \R^D$ is any vector such that $\theta^* + \delta \in N_{\theta^*}$.
\end{defn}

We note that this definition is actually different from the original definition given in \citet{li:2015:sparsistency}, who prove the two to be equivalent in their Proposition 3.1. \citet{li:2015:sparsistency} go on to prove bounds on the LSSC constants for linear and logistic regression, which we state as \cref{linearRegressionLSSC} and \cref{logisticRegressionLSSC} below.

Note that \cref{assum:LSSC} in the main text states that the LSSC holds with $N_{\theta^*} = \R^D$. We state \cref{assum:LSSC} in this form purely for conciseness; we will only consider checking \cref{assum:LSSC} for linear and logistic regression, both of which satisfy the LSSC with $N_{\theta^*} = \R^D$. Going beyond these cases, it is easily possible to state a version of our results with $N_{\theta^*} \neq \R^D$; however, this will require an extra assumption along the lines of Condition 7 of Theorem 5.1 in \citet{li:2015:sparsistency}, which is trivially satisfied when $N_{\theta^*} = \R^D$. In order to avoid stating an extra assumption that is trivially satisfied in the cases we consider, we chose to simply state the LSSC with $N_{\theta^*} = \R^D$.

\subsection{\cref{assum:LSSC,assum:lambdaMin,assum:incoherence,assum:boundedGradient,assum:lambdaSmall} imply $\supp\thetan \subseteq S$ for all $n$} \label{app-supportInclusion}

Theorem 5.1 of \citet{li:2015:sparsistency} gives conditions on $F$ under which $\supp\hat\theta = S$. So, if these conditions hold for all $\objn$, then we have $\supp\thetan = S$ for all $n$. Their Theorem 5.1 actually has two extra assumptions beyond \cref{assum:LSSC,assum:lambdaMin,assum:incoherence,assum:boundedGradient,assum:lambdaSmall}. The first is their Assumption 7; however, this is immediately implied by the fact that we assume the LSSC holds with $N_{\theta^*} = \R^D$. The second is their analogue of our \cref{assum:thetaMin-app}; however, they use this condition to imply that $\hat\theta = S$ after having shown that $\hat\theta \subseteq S$.

\subsection{Useful results for proving \cref{linearRegressionTheorem,logisticRegressionTheorem}}

Before going on to \cref{linearRegressionTheorem,logisticRegressionTheorem}, we will give a few useful results. We first define a sub-Exponential random variable:
\begin{defn}[\cite{vershynin:2017:hdpBook}]
  A random variable $V$ is $c_x$-sub-Exponential if $E[\exp(V / c_x)] \leq 2$.
\end{defn}
We will frequently use the fact that if $X$ is $c_x$-sub-Gaussian, then $X^2$ is $c_x^2$-sub-Exponential. Now we state a few existing results about the maxima of sub-Gaussian and sub-Exponential random variables that will be useful in our proofs.

\begin{lm}[Lemma 5.2 from \citet{handel:2016:lectureNotes}] \label{vhGeneralMax}
  Suppose that we have real valued random variables $Z_1, \dots, Z_N$ that satisfy $\log E[e^{\lambda Z_n }] \leq \psi(\lambda)$ for all $n=1,\dots, N$ and all $\lambda \geq 0$ for some convex function $\psi : \R \to \R$ with $\psi(0) = \psi'(0) = 0$. Then for any $u \geq 0$ :
  $$ \Pr\left[ \max_{n=1,\dots,N} Z_n \geq \psi^{*-1}(\log N + u) \right] \leq e^{-u}. $$
  where $\psi^{*-1}$ is the inverse of the Legendre dual of $\psi$.
\end{lm}

Remembering the definition of a sub-Gaussian random variable from \cref{defineSubGaussian}, \cref{vhGeneralMax} can be used to show the following:
\begin{cor} \label{subEsubGMax}
  Let $Z_1, \dots, Z_N$ be i.i.d.\ sub-Gaussian random variables with parameter $c_x$. Then:
  \begin{align}
     &\Pr \left[ \max_{n=1, \dots, N} Z_n \geq E[Z_n] + \sqrt{2Cc_x^2 \log N} + u \right] \leq e^{- C\frac{u^2}{2c_x^2}} \\
    &\Pr \left[ \max_{n=1, \dots, N} Z_n^2 \geq E[Z_n^2] + Cc_x^2 (\log N + 1 + u) \right] \leq e^{-u}
  \end{align}
\end{cor}
\begin{proof}
  For the first inequality, the definition of a sub-Gaussian random variable is that $\log E e^{\lambda Z_n} \leq \lambda^2 c_x / 2 =: \psi(\lambda)$, which has $\psi^*(y) = y^2/(2 C c_x^2)$ and $\psi^{*-1}(x) = \sqrt{2Cc_x^2 x}$. We use the upper bound:
  $$ \psi^{*-1}(\log N + u) = \sqrt{2Cc_x^2 (\log N + u)} \leq \sqrt{2Cc_x^2 \log N} + \sqrt{2Cc_x^2 u},$$
  Using this upper bound with \cref{vhGeneralMax} and changing variables $u \mapsto u^2 /(2Cc_x^2)$ gives the first inequality.

For the second inequality, use the fact that $Z_n^2$ is sub-Exponential with parameter $c_x^2$ so that it satisfies $\log Ee^{\lambda Z_n^2} \leq \psi(\lambda),$ where:
$$ \psi(\lambda) :=
\begin{cases}
  \lambda Cc_x^2 , & 0 \leq t \leq 1/c_x^2 \\
  \infty , & \mathrm{o.w.}
\end{cases}.
$$
  For $x \geq 0$, this $\psi$ has inverse Legendre dual $\psi^{*-1}(x) = Cc_x^2 (x+1)$. Plugging into \cref{vhGeneralMax} gives the result.
\end{proof}

\begin{prop} \label{normBound}
  Let $x_1, \dots, x_N$ be random vectors in $\R^D$ with i.i.d.\ $c_x$-sub-Gaussian components and $E[x_{nd}^2] = 1$. Then:
  \begin{equation} \Pr\left[ \max_{n=1,\dots,N} \n{x_n}_2 \geq \sqrt{D} + \sqrt{2C c_x^4 \log N} + u \right] \leq e^{-C \frac{u^2}{2c_x^4}}, \end{equation}
  where $C>0$ is some global constant, independent of $c_x, D,$ and $N$.
\end{prop}
\begin{proof}
  From Theorem 3.1.1 of \citet{vershynin:2017:hdpBook}, we have that $\n{x_n}_2 - \sqrt{D}$ is sub-Gaussian with parameter $C c_x^2$, where $C$ is some constant. Using the first part of \cref{subEsubGMax} gives the result.
\end{proof}

\subsection{Proof of \cref{linearRegressionTheorem} (Linear Regression)}

Recall \cref{assum:randomX,assum:linearRegressionY}: we assume a linear regression model $y_n = x_n^T \theta^* + \eps_n$, where $x_n \in \R^D$ has i.i.d.\ $c_x$-sub-Gaussian components with $E[x_{nd}^2] = 1$ and $\eps_n$ is $c_\eps$-sub-Gaussian. For notation throughout this section, we will let $C$ denote an absolute constant independent of any aspect of the problem ($N,D,\deff,c_x,$ or $c_\eps$) that will change from line to line (e.g.\ we may write $5C^2 = C$). We will frequently use $X_{\cdot, S}$ to denote the $N \times \deff$ matrix formed by taking the columns of $X$ that are in $S$, $x_{nS}$ to denote the coordinates of the $n$th vector of covariates $x_n$ that are in the set $S$, and $\xsn$ to denote the matrix $X_{\cdot, S}$ with the $n$th row removed. We will show the following theorem, stated more concisely as \cref{linearRegressionTheorem} in the main text:
\begin{thm}[Restated version of \cref{linearRegressionTheorem} from main text] \label{linearRegressionTheorem-app}
  Take \cref{assum:randomX,assum:incoherenceAssumption,assum:linearRegressionY,assum:DDeffScaling,assum:thetaMin-app}. Suppose the regularization parameter $\lambda$ satisfies:
  \begin{equation}
    \lambda \geq \frac{1}{\alpha - \MJLin}  \sqrt{ \frac{c_x^2 c_\eps^2 \log D}{NC} + \frac{25c_x^2 c_\eps^2}{NC}} +  \frac{4c_x c_\eps (\log(ND) + 26)}{N(\alpha - \MJLin)},
   \end{equation}
  where $C$ is a constant in $N,D,\deff, c_x$ and $c_\eps$, and $\MJLin$ is defined as:
  \begin{align}
  &\MJLin =  \nonumber \\
    &\quad \frac{C\deff \left( \sqrt{50c_x^2 } + \sqrt{2c_x^2  \log(N(D - \deff))}\right) \left( \sqrt{\deff} + \sqrt{50c_x^4} + \sqrt{2c_x^4 \log N} \right)}{N - 3c_x^2 \sqrt{N}\big(\sqrt{\deff } + 5\big)} + \nonumber \\
  &\quad \frac{C\deff \left(\deff + \deff c_x^2 (\log N + 26)\right) \left( \sqrt{N} + \sqrt{50c_x^4} + \sqrt{2c_x^4\log(D-\deff )}\right)\left(\sqrt{N\deff } + \sqrt{50c_x^4} \right)}{\big(N - 3c_x^2 \sqrt{N}\big(\sqrt{\deff } + 5\big)\big)^2} \label{MJDef}
\end{align}
  Then for $N$ sufficiently large, \cref{condition:supportStable} holds with probability at least $1 - 26e^{-25}$, where the probability is over the random data $\{(x_n, y_n)\}_{n=1}^N$.
\end{thm}

\begin{proof}
  For a fixed regularization parameter $\lambda$ and random data $\{x_n, y_n\}_{n=1}^N$, we are interested in the probability that any of \cref{assum:LSSC,assum:lambdaMin,assum:incoherence,assum:boundedGradient,assum:lambdaSmall} are violated, as \cref{prop-supportRecovery} then proves the result. For convenience in writing the incoherence condition, define $J_{nd} \in \R^D$, for $d \in S^c$, as:
  \begin{equation} J_{nd} := \left(\xsn^T \xsn\right)\inv \xsn^T X_{\bn, d}.  \label{JndDef} \end{equation}
  It is easiest to show that each of \cref{assum:LSSC,assum:lambdaMin,assum:incoherence,assum:boundedGradient,assum:lambdaSmall} hold with high probability separately, rather than all together, so we apply a union bound to get:
\begin{align*}
  \Pr& \left[ \mathrm{any\ assumption\ violated} \right] \leq \\
  & \Pr\left[ \min_n \lambda_{min}( \xsn^T \xsn ) = 0\right] \\
  & + \Pr\left[ \max_n \max_{d \in S^c} \n{J_{nd}}_1 \geq 1 \right] \\
  & + \Pr\left[ \max_n \n{\nabla F_{\bn}}_\infty > \frac{\lambda \big(1 - \max_n \max_{d\in S^c} \n{J_{nd}}_1 \big)}{4}\right] \\
  & + \Pr\left[ \frac{\min_n \lambda_{min}^2 (\xsn^T \xsn)}{4 \left( \big(1 - \max_n \max_{d\in S^c} \n{J_{nd}}_1\big) + 4 \right)^2} \frac{\big(1 - \max_n \max_{d\in S^c} \n{J_{nd}}_1\big)}{\deff K} \leq \lambda \right]
\end{align*}
We will bound each term by appealing to the Lemmas and Propositions proved below. Using \cref{linearRegressionLambdaMin} and \cref{linearRegressionBG}, the first and third terms are bounded by $16^{-25}$. As noted in \cref{linearRegressionLSSC}, we have $\Pr [K = 0] = 1$, so the final probablity is equal to zero (as the event reduces to $\infty < \lambda$). To bound the second probability, we have that \cref{linearRegressionIncohBound} says that:
$$ \Pr\left[ \max_n \max_{d \in S^c} \n{J_{nd}}_1 \geq 1 - \alpha + \MJLin \right] \leq 9e^{-25}. $$
As $\alpha > 0$, if  $\MJLin = o(1)$ as $N \to \infty$, we will we have that $1-\alpha + \MJLin < 1$ for large enough $N$. This would imply the third probability is $\leq 9e^{-25}$ for $N$ large enough. Under our conditions on the growth of $\deff$ and $D$, we can show that $\MJLin = o(1)$. We have, hiding constants and lower order terms in $N,D,$ and $\deff$:
\begin{align}
  \MJLin &= O\left( \frac{\deff \sqrt{\log(N) + \log (D)} \left( \sqrt{\deff} + \sqrt{\log(N)}\right)}{N-\sqrt{N\deff}} + \frac{\deff^{5/2} \log(N) \left(\sqrt{N} + \sqrt{\log(D)}\right) \sqrt{N}}{(N - \sqrt{N\deff})^2} \right) \nonumber \\
  &= O\left( \frac{\deff \left(\sqrt{\deff\log(N)} + \sqrt{\deff \log(D)} + \log(N) + \sqrt{\log(N) \log(D)}\right)}{N - \sqrt{N\deff}} + \frac{\deff^{5/2} N\log(N)}{(N - \sqrt{N\deff})^2} \right), \label{bigOMJ}
\end{align}
where the second statement follows from using $\sqrt{\log(N) + \log(D)} \leq \sqrt{\log(N)} + \sqrt{\log(D)}$ and $D = o(e^N)$. Now, given that $\deff = o([N/\log(N)]^{2/5})$, the second term in \cref{bigOMJ} is $o(1)$. The first term is also $o(1)$ by combining $\deff = o([N/\log(N)]^{2/5})$ with $\deff^{3/2}\sqrt{\log(D)} = o(N)$. Thus, $\MJLin = o(1)$, which completes the proof.
\end{proof}

What remains is to prove \cref{linearRegressionLambdaMin,linearRegressionBG,linearRegressionIncohBound,linearRegressionLSSC} needed to prove \cref{linearRegressionTheorem-app}. We do this in the following four subsections.

\subsection{Linear regression: minimum eigenvalue}

All we want to bound right now is the probability that the minimum eigenvalue is actually equal to zero; however, it will be useful later to show that it is $\Omega(N)$ with high probability. The lemma we prove in this section shows exactly this. We will start with two propositions.

\begin{prop}
  If $X_{\cdot, S}$ is an $N \times \deff $ matrix with independent $c_x$-sub-Gaussian entries with unit second moments, then:
  \begin{equation} \Pr \left[ \lambda_{min} (X_{\cdot, S}^T X_{\cdot, S}) \leq N - 2Cc_x^2 \sqrt{N}(\sqrt{\deff} + 5) \right] \leq 2e^{-25}, \end{equation}
  where $C > 0$ is a global constant.
\end{prop}
\begin{proof}
  Theorem 4.6.1 of \citet{vershynin:2017:hdpBook} gives a concentration inequality for the minimum singular value, $s_{min}(X_{\cdot, S})$, of $X_{\cdot, S}$:
  \begin{equation} \Pr\left[ s_{min}(X_{\cdot, S}) \leq \sqrt{N} - Cc_x^2 (\sqrt{\deff} + t)\right] \leq 2e^{-t^2}. \end{equation}
  Using the fact that the minimum eigenvalue of $X_{\cdot, S}^TX_{\cdot, S}$ is the square of the minimum singular value of $X_{\cdot, S}$ and putting in $t = 5$:
  $$ \Pr\left[ \lambda_{min} (X_{\cdot, S}^T X_{\cdot, S}) \leq N -2Cc_x^2 \sqrt{N}(\sqrt{\deff} + 5) + C^2c_x^4 (\sqrt{\deff} + 5)^2 \right] \leq 2e^{-25}. $$
  Dropping the $C^2 c_x^4(\sqrt{\deff} + 5)^2$ gives the result.
\end{proof}

\begin{prop} \label{fullLambdaMinBound}
  If $\xsn$ is the $N-1 \times \deff$ matrix formed by removing the $n$th row from $X_{\cdot, S}$, we have:
  \begin{equation} \lambda_{min}(\xsn^T \xsn) \geq \lambda_{min}(X_{\cdot, S}^T X_{\cdot, S}) - \n{x_{nS}}_2^2, \end{equation}
  where $x_n$ is the $n$th row of $X_{\cdot, S}$.
\end{prop}
\begin{proof}
  Looking at the variational characterization of the minimum eigenvalue:
  \begin{align*}
    \lambda_{min}(\xsn^T \xsn) &= \min_{z \in \R^{\deff} \, : \, \n{z}_2 = 1} \bigg[ z^T X_{\cdot, S}^T X_{\cdot, S} z - z^T x_{nS} x_{nS}^T z \bigg] \\
    & \geq \min_{z} z^T X_{\cdot, S}^T X_{\cdot, S} z -  \max_z z^T x_{nS} x_{nS}^T z \\
    & = \lambda_{min}(X_{\cdot, S}^T X_{\cdot, S}) - \n{x_{nS}}_2^2.
  \end{align*}
\end{proof}

The above two propositions now allow us to prove the bound we want on $\min_n \lambda_{min}(\xsn^T \xsn)$. In the following lemma, we will assume that $\deff = o(N/\log(N))$. While we ultimately will have the more restrictive requirement that $\deff = o([N/\log(N)]^{2/5})$ in \cref{assum:DDeffScaling}, the current result can be stated with the less restrictive requirement of $o(N/\log(N))$.

\begin{lm} \label{linearRegressionLambdaMin}
  Suppose $X_{\cdot, S}$ is a $N \times \deff$ matrix with independent $c_x$-sub-Gaussian entries and $\deff$ is $o(N/\log(N))$ as function of $N$. Then we have for $N$ sufficiently large:
  \begin{equation} \Pr\left[ \min_{n=1,\dots N} \lambda_{min}(\xsn^T \xsn ) \leq N - 3Cc_x^2 \sqrt{N}\big( \sqrt{\deff} + 5\big) \right] \leq 3e^{-25} \label{eq-linearRegressionLambdaMin}\end{equation}
\end{lm}
\begin{proof}
 In what follows, and repeatedly throughout the rest of our proofs, we will make use of the following generic inequality for any events $A$ and $B$:
  \begin{align}
    \Pr[A] &= \Pr[A\mid B] \Pr[B] + \Pr[A \mid B^c] \Pr[B^c] \nonumber \\
    &\leq \Pr[A\mid B] Pr[B] + \Pr[B^c], \label{conditioningTrick}
  \end{align}
   Calling the probability on the left hand side of \cref{eq-linearRegressionLambdaMin} $P$, we can break $P$ down as, for some constant $L_{min}$:
  \begin{align*}
    P &\leq \\
    &\Pr\left[ \min_{n=1,\dots N} \lambda_{min}(\xsn^T \xsn ) \leq  N - 3Cc_x^2 \sqrt{N} \big( \sqrt{\deff} + 5 \big) \;\;\bigg\vert\;\; \lambda_{min}(X_{\cdot, S}^T X_{\cdot, S}) \geq L_{min} \right] \Pr\left[ \lambda_{min}(X_{\cdot,S}^T X_{\cdot, S}) \geq L_{min} \right] \\
    & \quad\quad\quad\quad\quad\quad + \Pr\left[\lambda_{min}(X_{\cdot, S}^T X_{\cdot, S}) \leq L_{min} \right] \\
    &\leq \Pr\left[\min_{n=1,\dots N} L_{min} - \n{x_{nS}}_2^2 \leq  N - 3Cc_x^2 \sqrt{N} \big(\sqrt{\deff} + 5\big) \;\;\bigg\vert\;\; \lambda_{min}(X_{\cdot, S}^T X_{\cdot, S}) \geq L_{min} \right] \Pr\left[ \lambda_{min}(X_{\cdot,S}^T X_{\cdot, S}) \geq L_{min} \right] \\ 
    & \quad\quad\quad\quad\quad\quad + \Pr\left[\lambda_{min}(X_{\cdot, S}^T X_{\cdot, S}) \leq L_{min} \right] \\
    &\leq \Pr\left[ \max_n \n{x_{nS}}^2_2 \geq L_{min} -  N + 3Cc_x^2 \sqrt{N} \big( \sqrt{\deff} + 5\big) \right] + \Pr\left[ \lambda_{min}(X_{\cdot, S}^T X_{\cdot, S}) \leq L_{min} \right]
  \end{align*}
  Picking $L_{min} = N - 2Cc_x^2 \sqrt{N}\big(\sqrt{\deff} + 5\big)$, we have that the second probability at most $2e^{-25}$ by \cref{fullLambdaMinBound}. Now to control the $\max_n \| x_{nS}\|_2^2$, note that $\|x_{nS}\|_2^2$ is $\deff c_x^2$-sub-Exponential, and choose $u = 25$ in the second statement of \cref{subEsubGMax}; this tells us that the first probability is at most $e^{-25}$ if $E[\n{x_{nS}}_2^2] + Cc_x^2(\log N + 26) = \deff + Cc_x^2(\log N + 26)$ is less than $Cc_x^2 \sqrt{N}\big( \sqrt{\deff} + 5\big)$, which, for $\deff$ being $o(N/\log(N))$, is satisfied for $N$ large enough.
\end{proof}

\subsection{Linear regression: incoherence}

The following proposition will be useful in proving \cref{linearRegressionIncohBound} below:

\begin{prop} \label{lemmaIncoherence}
  Let $z \in \R^N$ be any vector and $z_{\backslash n} \in \R^{N-1}$ the same vector with the $n$th coordinate removed. Also let $X_{\cdot, S} \in \R^{N \times \deff}$ be some matrix with $\xsn$ the same matrix with the $n$th row removed. Define, for any vector $z \in R^N$:
  \begin{equation} J_{nz} := \left(\xsn^T \xsn\right)\inv \xsn^T  z_{\backslash n}, \label{defJ} \end{equation}
  and $J_z$ the same but with no row removed. Then:
  \begin{align*}
    \n{J_{nz} - J_z}_1 \leq & \deff \frac{\abs{z_n} \n{x_{nS}}_2}{\lambda_{min}\big( \xsn^T \xsn \big)} \\
    & + \deff \frac{\n{x_{nS}}_2^2}{\lambda_{min}^2 \big( \xsn^T \xsn \big)} \n{z}_2 \n{X_{\cdot, S}}_2, %\sum_{m=1}^N \abs{z_m} \n{x_m}_2
  \end{align*}
  where $\n{X_{\cdot, S}}_2 := \sqrt{\sum_{n=1}^N \sum_{s \in S} X_{ns}^2}$.
\end{prop}
\begin{proof}
  We can rewrite $J_z = (X_{\cdot, S}^T X_{\cdot, S})\inv X_{\cdot, S}^T z$ by noting that $X_{\cdot, S}^T X_{\cdot, S}$ and $\xsn^T \xsn$ differ by a rank one update and then applying the Sherman-Morrison formula:
\begin{align} J_z &= (X_{\cdot, S}^T X_{\cdot, S})\inv X_{\cdot, S}^T z \\
  &= \left( (\xsn^T \xsn)\inv - \frac{(\xsn^T \xsn)\inv x_{nS} x_{nS}^T (\xsn^T \xsn)\inv}{1 + x_{nS}^T (\xsn^T \xsn)\inv x_{nS}} \right) X_{\cdot, S}^T z \\
  &= \left( J_{nz} + (\xsn^T \xsn)\inv x_{nS} z_n \right) - \frac{(\xsn^T \xsn)\inv x_{nS} x_{nS}^T (\xsn^T \xsn)\inv}{1 + x_{nS}^T (\xsn^T \xsn)\inv x_{nS}} X_{\cdot, S}^T z
\end{align}
To cleanup notation a bit, let $B := \xsn^T \xsn$. We can continue to rewrite the above as:
\begin{align}
  &= \left( J_{nz} + B\inv x_{nS} z_n \right) - \frac{B\inv x_{nS} }{1 + x_{nS}^T B\inv x_{nS}} \sum_{m=1}^N z_m x_{nS}^T B\inv x_{mS}
\end{align}
Now, we are interested in $\n{J_{nz} - J_z}_1$, which we will bound by subtracting $J_{nz}$ from both sides of the above equation and then examine each coordinate by multiplying by the $i$th unit vector $e_i$:
\begin{align} \abs{e_i^T(J_{nz} - J_z)} &\leq \abs{e_i^T B\inv x_{nS}}\abs{z_n} + \frac{\abs{e_i^T B\inv x_{nS}}}{1 + x_{nS}^T B\inv x_{nS}} \sum_{m=1}^N \abs{z_m} \abs{x_{nS}^T B\inv x_{mS}} \\
  &\leq \abs{z_n} \lambda_{max}(B\inv) \n{x_{nS}}_2 + \frac{\lambda^2_{max}(B\inv) \n{x_{nS}}_2^2}{1 + \lambda_{min}(B\inv) \n{x_{nS}}^2_2} \sum_{m=1}^N \abs{z_m} \n{x_{mS}}_2
\end{align}
The $\lambda_{min}(B\inv) \n{x_{nS}}_2^2$ is strictly positive, so we can drop it from the denominator for a further upper bound. Using the fact that, for the positive semidefinite matrix $B$ we have $\lambda_{min}(B\inv) = 1/\lambda_{max}(B)$ and $\lambda_{max}(B\inv) = 1/\lambda_{min}(B)$, we get:
\begin{equation} \abs{e_i^T(J_{nz} - J_z)} \leq \frac{\abs{z_n} \n{x_{nS}}_2}{\lambda_{min}(B)} + \frac{\n{x_{nS}}_2^2}{\lambda^2_{min}(B)}  \sum_{m=1}^N \abs{z_m} \n{x_{mS}}_2. \end{equation}
Finally, use Cauchy-Schwarz to get $\sum_{m=1}^N \abs{z_m} \n{x_{mS}}_2 \leq \n{z}_2 \n{X_{\cdot, S}}_2$, where $\n{X_{\cdot,S}}_2 := \left(\sum_{m=1}^N \sum_{s \in S} x_{ms}^2\right)^{1/2}$. Notice that our upper bound is now independent of the index $i$; this means we have a bound on any coordinate $i$ of $\abs{(J_{nz} - J_z)}$. So, multiplying this bound by $\deff$ upper bounds $\n{J_{nz} - J_z}_1$, which gives the result.
\end{proof}

To get a high probability upper bound on $\n{J_{nd}}_1$, the idea will be to use $\| J_{nd}\|_1 \leq \| J_d \|_1 + \| J_{nd} - J_d\|_1$, and then put high probability bounds on the bound given by \cref{lemmaIncoherence}.

\begin{lm} \label{linearRegressionIncohBound}
  Take \cref{assum:randomX,assum:incoherenceAssumption,assum:linearRegressionY}. Then, for the scalar $\MJLin$ defined in \cref{linearRegressionTheorem}, we have:
  \begin{equation} \Pr\left[ \max_{n=1,\dots, N} \max_{d \in S^c} \n{J_{nd}}_1 \geq 1 - \alpha + \MJLin \right] \leq 10e^{-25}, \end{equation}
  where $J_{nd}$ is defined in \cref{JndDef} above.
\end{lm}
\begin{proof}
  First, for any $n$ and $d$, we have $\| J_{nd}\|_1 \leq \| J_d \|_1 + \| J_{nd} - J_d\|_1$. We can upper bound $\|J_{nd} - J_d\|_1$ using \cref{lemmaIncoherence} and then apply a high probability upper bound. Following the same idea of conditioning and peeling off terms as in the proof of \cref{linearRegressionLambdaMin}, we can condition on the following events, the complement of each of which has a small constant probability:
\begin{align}
  &\set{\min_n \lambda_{min}(\xsn^T \xsn) \geq N - 3Cc_x^2 \sqrt{N}\big(\sqrt{\deff} + 5\big)} \label{linearRegressionLambdaMinEvent}\\
  &\set{\n{X_{\cdot,S}}_2 \leq \sqrt{N\deff} + \sqrt{50Cc_x^4}} \\
  &\set{\max_n \n{x_{nS}}_2 \leq \sqrt{\deff} + \sqrt{50Cc_x^4} + \sqrt{2Cc_x^4 \log N}} \\
  &\set{\max_{d \in S^c} \n{X_{\cdot, d}}_2 \leq \sqrt{N} + \sqrt{50Cc_x^4} + \sqrt{2Cc_x^4 \log(D - \deff)}} \\
  &\set{\max_{n} \max_{d \in S^c} \abs{x_{n, d}} \leq \sqrt{50Cc_x^2} + \sqrt{2c_x^2 \log(N(D - \deff))}} \\
  &\set{\max_n \n{x_{nS}}_2^2 \leq \deff + c_x^2 \deff (\log N + 26)} 
\end{align}

The probability of the complement of the first event is $\leq 3e^{-25}$ by \cref{linearRegressionLambdaMin}, the second is $\leq e^{-25}$ by noting that $\n{X_{\cdot,S}}_2 - \sqrt{N\deff}$ is a $Cc_x^2$-sub-Gaussian random variable and applying a standard sub-Gaussian bound, the third is $\leq e^{-25}$ by applying \cref{normBound}, the fourth is $\leq e^{-25}$ by the same reasoning as the third, and the fifth is $\leq 2e^{-25}$ by the first part of \cref{subEsubGMax}. Finally, the sixth is $\leq e^{-25}$ by noting that $\| x_{nS}\|_2^2$ is a $c_x^2 \deff$-sub-Exponential random variable, to which we can apply \cref{subEsubGMax}. All in all, these probabilities sum up to $9e^{-25}$. Conditioned on all these events, we can upper bound the upper bound on $\n{J_{nd} - J_d}_1$ given by \cref{lemmaIncoherence} to get:
\begin{align*}
  &\n{J_{nd} - J_d}_1 \leq \\
  &\quad \frac{C\deff \left( \sqrt{50c_x^2 } + \sqrt{2c_x^2  \log(N(D - \deff))}\right) \left( \sqrt{\deff} + \sqrt{50c_x^4} + \sqrt{2c_x^4 \log N} \right)}{N - 3c_x^2 \sqrt{N}\big(\sqrt{\deff } + 5\big)} + \nonumber \\
  &\quad \frac{C\deff \left(\deff + \deff c_x^2 (\log N + 26)\right) \left( \sqrt{N} + \sqrt{50c_x^4} + \sqrt{2c_x^4\log(D-\deff )}\right)\left(\sqrt{N\deff } + \sqrt{50c_x^4} \right)}{\big(N - 3c_x^2 \sqrt{N}\big(\sqrt{\deff } + 5\big)\big)^2}
\end{align*}
Call the entire quantity on the right-hand side of this inequality $\MJLin$, and call the union of the above six events the event $F$. Then by conditioning on F and the event $\set{\max_{d\in S^c} \| J_{d}\|_1 < 1- \alpha}$, we get:
\begin{align}
  &\Pr\left[ \max_{n \in [N]} \max_{d \in S^c} \n{J_{nd}}_1 \geq 1 - \alpha +  \MJLin \right] \leq \\
  &\quad\quad \Pr\left[ \max_n \max_{d \in S^c} \n{J_{nd} - J_n}_1 \geq \MJLin  \mid F \right] + \Pr\left[F^c\right] + \Pr\left[ \max_{d \in S^c} \n{J_{d}}_1 \geq 1-\alpha\right]
\end{align}
By the definition of $\MJLin$ above, we know that the first probability is zero, by the argument above and a union bound we know $\Pr[F^c] \leq 9e^{-25}$, and the third is $\leq e^{-25}$ by \cref{assum:incoherenceAssumption}.
\end{proof}

\subsection{Linear regression: bounded gradient}

We need to bound the probability
\begin{align*}
  &\Pr \left[ \max_{n \in [N]} \n{\nabla \objn (\theta^*)}_\infty \geq \frac{\lambda \big(1 - \max_n \max_{d \in S^c} \n{J_{nd}}_1\big)  }{4} \right] \\
  \leq &\Pr\left[ \max_{n \in [N]} \left( \n{\nabla F(\theta^*)}_\infty + \n{\frac{1}{N} \nabla f(x_n^T\theta^*, y_n)}_\infty \right) \geq \frac{\lambda \big(1 -\max_n \max_{d \in S^c} \n{J_{nd}}_1\big) }{4} \right]
\end{align*}
Conditioning on the event that $\| \nabla F(\theta^*)\|_\infty \leq B_G$ for some number $B_G$ and the event that $\max_n \max_{d \in S^c} \n{J_{nd}}_1 \leq 1-\alpha + \MJLin$, we get that this probability is less than or equal to:
\begin{align}
  \leq &\Pr \left[ \max_{n=1,\dots, N} \n{\frac{1}{N} \nabla f(x_n^T \theta^*, y_n)}_\infty \geq \frac{\lambda (\alpha - \MJLin) }{4} - B_G \right] \nonumber \\
  & \;\;\;\;\;\;\;\;\;\; + \Pr\left[ \n{\nabla F(\theta^*)}_\infty \geq B_G \right] + \Pr\left[ \max_n \max_{d \in S^c} \n{J_{nd}}_1 \geq 1-\alpha + \MJLin \right] \label{linearRegressionInitBGBound}
\end{align}

The following proposition gives a reasonable value for $B_G$:
\begin{prop} \label{linearRegressionFullDataBG}
  In the above setup for linear regression,
  \begin{equation} \Pr\left[ \n{\nabla F(\theta^*)}_\infty \geq \left[\frac{c_x^2 c_\eps^2 \log D}{NC} + \frac{25 c_x^2 c_\eps^2}{NC} \right]^{1/2} \right] \leq e^{-25} \end{equation}
\end{prop}

\begin{proof}
  The $d$th coordinate of the gradient is $(\nabla F(\theta^*))_d = 1/N \sum_n \eps_n x_{nd}$. First, we have that $1/N \sum_n \eps_n x_{nd}$ is a $c_x c_\eps$-sub-Exponential random variable. By Bernstein's inequality (see Theorem 2.8.1 from \citet{vershynin:2017:hdpBook}), we have:
  $$ \Pr\left[ \frac{1}{N} \abs{\sum_{n=1}^N \eps_n x_{nd}} \geq \left[\frac{c_x^2 c_\eps^2 \log D}{NC} + \frac{25 c_x^2 c_\eps^2}{NC} \right]^{1/2} \right] \leq e^{-25 - \log D} $$
  If we union bound over the $D$ dimensions of $\nabla F(\theta^*)$, we get that the probability in the proposition's statement is $\leq De^{-25 - \log D}= e^{-25}$, as claimed.  
\end{proof}

Now we can prove the lemma we need, which bounds the probability that any $\n{\nabla \objn(\theta^*)}_\infty$ is large:

\begin{lm} \label{linearRegressionBG}
  For the above setup for linear regression and the $\lambda$ given in \cref{linearRegressionTheorem}, we have:
  \begin{equation} \Pr \left[ \max_{n=1,\dots, N} \n{\nabla \objn(\theta^*)}_\infty \geq \frac{\lambda \big(1 - \max_n \max_{d \in S^c} \n{J_{nd}}_1 \big) }{4} \right] \leq 13e^{-25} \end{equation}
\end{lm}
\begin{proof}
  We can first apply the bound worked out in \cref{linearRegressionInitBGBound}. Picking $B_G$ to be the value given in \cref{linearRegressionFullDataBG}, the second probability is $\leq e^{-25}$ by \cref{linearRegressionFullDataBG}, and the third is $\leq 10e^{-25}$ by \cref{linearRegressionIncohBound}. To analyze the first probability, note that we can write the event as:
  $$ \Pr\left[\frac{1}{N} \max_n \max_d \abs{\eps_n x_{nd}} \geq \frac{\lambda (\alpha - \MJLin) }{4} - B_G \right]. $$
  Looking at the form of $\lambda$ given in \cref{linearRegressionTheorem}, we get that this is equal to:
  $$ = \Pr\left[\frac{1}{N} \abs{\max_n \max_d \eps_n x_{nd}} \geq 4c_x c_\eps (\log(ND) + 26) \right]. $$
  The event we're considering is just the absolute value of the max of $ND$ sub-Exponential variables with parameter $c_x c_\eps$. Plugging into \cref{subEsubGMax} gives that this probability is $\leq 2e^{-25}$.
\end{proof}

\subsection{Linear regression: $\lambda$ small enough}

To check the bound in \cref{assum:lambdaSmall}, we need to know the LSSC constant $K$ for linear regression:

\begin{prop}[\citep{li:2015:sparsistency}] \label{linearRegressionLSSC}
  For the linear regression setup in \cref{linearRegressionTheorem-app}, the loss $F(\theta)$ satisfies the $(\theta^*, N_{\theta^*})$ LSSC with constant $K=0$ for any $\theta^*$, $N_{\theta^*}$, and any data $X,Y$.
\end{prop}
\begin{proof}
  This follows from the fact that $F(\theta) = \frac{1}{2} \n{X\theta - Y}_2^2$ has zero third derivatives, implying that $D^3F(\theta)[u,u,e_j] = 0$ for any  $\theta, u \in \R^D$ and coordinate vector $e_j \in \R^D$.
\end{proof}

As linear regression has a LSSC constant $K$ that is deterministically equal to zero, the only constraint implied by the bound in \cref{assum:lambdaSmall} is that $\lambda < \infty$, which is always satisfied by the value of $\lambda$ given in \cref{linearRegressionTheorem-app}.
%\begin{equation}\Pr\left[ \frac{\min_n \lambda_{min}^2 (\xsn^T \xsn)}{4 \left(\max_n \max_{d\in S^c} \n{J_{nd}}_1 + 4 \right)^2} \frac{\max_n \max_{d\in S^c} \n{J_{nd}}_1}{\deff K} \leq \lambda \right] = 0 \end{equation}

%%%%%%%%%%%%%%%%%%%%%%%%%%%%%%%%%%%%%%%%%%%%%%%%%%%%%%%%%
%%%%%%%%%%%%%%%%%%%%%%%%%%%%%%%%%%%%%%%%%%%%%%%%%%%%%%%%%
% Logistic Regression
%%%%%%%%%%%%%%%%%%%%%%%%%%%%%%%%%%%%%%%%%%%%%%%%%%%%%%%%%
%%%%%%%%%%%%%%%%%%%%%%%%%%%%%%%%%%%%%%%%%%%%%%%%%%%%%%%%%

\subsection{Proof of \cref{logisticRegressionTheorem} (Logistic Regression)}

Recall \cref{assum:randomX,assum:logisticRegressionY}: we assume a logistic regression model such that the responses $y_n \in \set{-1,1}$ with $\Pr\left[y_n = 1\right] = 1/(1 + e^{-x_n^T\theta^*})$. The derivatives are slightly more complicated here than in the case of linear regression. In particular, defining:
\begin{equation} \dnonestar := \frac{-y_n}{1 + e^{y_n x_n^T \theta^*}}, \quad \dntwostar := \frac{e^{x_n^T\theta^*}}{(1 + e^{x_n^T\theta^*})^2}, \end{equation}
the derivatives of $F$ are:
\begin{equation} \nabla_\theta F(\theta^*) = \frac{1}{N} \sum_{n=1}^N \dnonestar x_n , \quad \nabla^2_\theta F(\theta^*) = \frac{1}{N} \sum_{n=1}^N \dntwostar x_n x_n^T. \end{equation}
For comparison, things were easier for linear regression because $\dntwostar = 1$ and $\dnonestar = \eps_n$ for some sub-Gaussian noise $\eps_n$. Still, we will be able to extend basically all our proof techniques for linear regression by using the fact that $\abs{\dntwostar}$ and $\abs{\dnonestar}$ are both $\leq 1$, allowing us to drop them in many of our upper bounds. This will allow us to prove a very similar result to \cref{linearRegressionTheorem}. Again, we will let $C$ denote an absolute constant independent of any aspect of the problem ($N,D,\deff,c_x,$) that will change from line to line (e.g.\ we may write $5C^2 = C$).

\begin{thm}
\label{logisticRegressionTheorem-app}
Take \cref{assum:randomX,assum:incoherenceAssumption,assum:logisticRegressionY,assum:DDeffScaling,assum:logisticRegressionLambdaMin,assum:thetaMin-app}. Suppose the regularization parameter is set as:
  \begin{align}
    \lambda \geq \frac{C}{\alpha - \MJLogr} \left( \sqrt{c_x^2 \frac{25 + \log D}{N}} + \frac{\sqrt{2c_x^2 \log(ND)} + \sqrt{50c_x^2}}{N} \right),
   \end{align}
  where $C$ is a constant in $N,D,$ and $c_x$, and $\MJLogr$ is defined similarly to $\MJLin$ from \cref{linearRegressionTheorem-app}, but with different denominators:
  \begin{align}
  &\MJLogr =  \nonumber \\
  &\quad  \frac{C \deff \left( \sqrt{50c_x^2} + \sqrt{2c_x^2  \log(N(D - \deff))}\right) \left( \sqrt{\deff} + \sqrt{50c_x^4} + \sqrt{2c_x^4 \log N} \right)}{L_{min} - c_x^2 \sqrt{N}\big(\sqrt{\deff } + 5\big)} + \nonumber \\
  &\quad \frac{C \deff \left(\deff + \deff c_x^2 (\log N + 26)\right) \left( \sqrt{N} + \sqrt{50c_x^4} + \sqrt{2c_x^4\log(D-\deff )}\right)\left(\sqrt{N\deff } + \sqrt{50c_x^4} \right)}{\big(L_{min} - c_x^2 \sqrt{N}\big(\sqrt{\deff } + 5\big)\big)^2} \label{MJLogr}
\end{align}
  Then for $N$ sufficiently large, \cref{condition:supportStable} holds with probability at least $1 - 43e^{-25}$, where the probability is over the random data $\set{(x_n, y_n)}_{n=1}^N$.
\end{thm}
\begin{proof}
  The proof is exactly the same as that of \cref{linearRegressionTheorem-app} -- we bound the probability that any of \cref{assum:LSSC,assum:lambdaMin,assum:incoherence,assum:boundedGradient,assum:lambdaSmall} are violated by a union bound -- except that we use \cref{logisticRegressionLambdaMin,logisticRegressionIncohBound,logisticRegressionBG,logisticRegressionLSSC} below to bound each term. Note that we have $\MJLogr = o(1)$ by \cref{assum:DDeffScaling,assum:logisticRegressionLambdaMin}.
\end{proof}

\subsection{Logistic regression: lambda min}

\begin{lm} \label{logisticRegressionLambdaMin}
Take \cref{assum:logisticRegressionLambdaMin}. Further suppose that $\deff$ grows as $o(N/\log(N))$. Then for $N$ sufficiently large:
\begin{equation} \Pr\left[\min_{n = 1,\dots, N}\lambda_{min}(\nabla_\theta^2 \objn(\theta^*)_{SS}) \leq L_{min} - Cc_x^2 \sqrt{N}\big(\sqrt{\deff} + 5 \big) \right] \leq 3e^{-25}, \end{equation}
where $L_{min}$ is the constant from \cref{assum:logisticRegressionLambdaMin}.
\end{lm}
\begin{proof}
  We have by \cref{fullLambdaMinBound} and the fact that $\abs{\dntwostar} \leq 1$ :
  \begin{align*}
    \lambda_{min}(\nabla_\theta^2 \objn(\theta^*)_{SS}) &\geq \lambda_{min}(\nabla_\theta^2 \obj(\theta^*)_{SS}) - \n{x_{nS}}_2^2 \abs{\dntwostar} \\
    &\geq \lambda_{min}(\nabla_\theta^2 \obj(\theta^*)_{SS}) - \n{x_{nS}}_2^2 .
  \end{align*} 
  The rest of the proof is now exactly the same as that of \cref{linearRegressionLambdaMin}.
\end{proof}

\subsection{Logistic regression: incoherence}

We can get exactly the same bound as in \cref{linearRegressionIncohBound}. To do so, we first note that \cref{lemmaIncoherence} is only written to deal with Hessians of the form $X^T X$; however, if we rewrite our data as $\bar x_n := \sqrt{\dntwostar} x_n$, the Hessian for logistic regression is equal to $\bar X^T \bar X$. We can further upper bound the upper bound in \cref{lemmaIncoherence} by noting that $\abs{\dntwostar} \leq 1 \implies \n{\bar x_n}_2 \leq \n{x_n}_2$. Applying this reasoning, we get an identical lemma to \cref{linearRegressionIncohBound}

\begin{lm} \label{logisticRegressionIncohBound}
  Take \cref{assum:randomX,assum:incoherenceAssumption,assum:logisticRegressionY,assum:logisticRegressionLambdaMin}. Then for the scalar $\MJLogr$ defined in \cref{logisticRegressionTheorem}, we have:
  \begin{equation} \Pr\left[ \max_{n=1,\dots, N} \max_{d \in S^c} \n{J_{nd}}_1 \geq 1 - \alpha + \MJLogr \right] \leq 10e^{-25}, \end{equation}
  where $J_{nd}$ is defined in \cref{JndDef}.
\end{lm}
\begin{proof}
  The proof is very similar to that of \cref{linearRegressionIncohBound}. To prove \cref{linearRegressionIncohBound}, we wrote $\|J_{nd}\|_1 \leq \| J_{d}\|_1 + \|J_{nd} - J_d\|_1$. To bound $\|J_d\|_1$ with high probability, we applied \cref{assum:incoherenceAssumption}. To bound $\|J_{nd} - J_d\|_1$, we used the bound from \cref{lemmaIncoherence}, and then conditioned on a number of high-probability events to give an overall bound. We can condition on all of the same events, except we replace the event in \cref{linearRegressionLambdaMinEvent} by:
  \begin{equation}
    \set{\min_{n=1,\dots, N} \lambda_{min}\left( \nabla_\theta^2 \objn_{SS} \right) \geq L_{min} - Cc_x^2 \sqrt{N} (\sqrt{\deff} + 5)}.
  \end{equation}
  By \cref{logisticRegressionLambdaMin}, the complement of this event has probability at most $3e^{-25}$. We condition on the rest of the events in the proof of \cref{linearRegressionIncohBound} and finish the proof along the same lines.
\end{proof}

\subsection{Logistic regression: bounded gradient}

Again, we are interested in bounding:
$$ \Pr\left[ \max_{n = 1, \dots, N} \n{\nabla\objn(\theta^*)}_\infty \geq \frac{\lambda \big(1 - \max_n \max_{d\in S^c} \n{J_{nd}}_1\big)}{4} \right] $$
The same reasoning that led to \cref{linearRegressionInitBGBound} gives us the same bound:
\begin{align}
  \leq &\Pr \left[ \max_{n=1,\dots, N} \n{\frac{1}{N} \nabla_\theta f(x_n^T \theta^*, y_n)}_\infty \geq \frac{\lambda (\alpha - \MJLogr) }{4} - B_G \right] \nonumber \\
  & \;\;\;\;\;\;\;\;\;\; + \Pr\left[ \n{\nabla F(\theta^*)}_\infty \geq B_G \right] + \Pr\left[ \max_n \max_{d \in S^c} \n{J_{nd}}_1 \geq 1-\alpha + \MJLogr \right] \label{logisticRegressionInitBGBound}
\end{align}

Just as in the case of linear regression, we can first pick a reasonable value for $B_G$:

\begin{prop} \label{logisticRegressionBGProp} For the logistic regression setup above, we have:
  \begin{equation}\Pr\left[ \n{\nabla F(\theta^*)}_\infty \geq c_x \sqrt{\frac{25 + \log D}{CN}} \right] \leq 2e^{-25}. \end{equation}
\end{prop}
\begin{proof}
  The $d$th coordinate of the gradient is $(\nabla F(\theta^*))_d = 1/N \sum_n \dnonestar x_{nd}$, where
  $$ \dnonestar = \frac{-y_n}{1 + e^{y_n x_n^T \theta^*}}.$$
  Noting that this satisfies $\abs{\dnonestar} \leq 1$:
  \begin{align*}
    \Pr &\left[ \lvert \frac{1}{N}\sum_{n=1}^N \dnonestar x_{nd} \rvert \geq c_x \sqrt{\frac{25 + \log D}{CN}} \right] \\
    &\leq \Pr\left[ \sum_{n=1}^N \abs{x_{nd}} \geq c_x \sqrt{N \frac{25 + \log D}{C}} \right] \\
    &\leq 2e^{-25 - \log D} ,
  \end{align*}
  where the final inequality comes from noting that $\abs{x_{nd}}$ is also $c_x$-sub-Gaussian and using Hoeffding's inequality (Theorem 2.6.2 from \citet{vershynin:2017:hdpBook}). Union bounding over all $D$ dimensions of $\nabla F(\theta^*)$ gives the result.
\end{proof}
\begin{lm} \label{logisticRegressionBG}
  For the above setup for logistic regression and the $\lambda$ given in \cref{logisticRegressionTheorem}, we have:
  \begin{equation} \Pr \left[ \max_{n \in [N]} \n{\nabla \objn(\theta^*)}_\infty \geq \frac{\lambda \big(1 - \max_n \max_{d \in S^c} \n{J_{nd}}_1 \big) }{4} \right] \leq 14e^{-25} \end{equation}
\end{lm}
\begin{proof}
  Just as in the proof of \cref{linearRegressionBG}, we will apply the upper bound in \cref{logisticRegressionInitBGBound} and then bound each term. The second probability in \cref{logisticRegressionInitBGBound} is $\leq 10e^{-25}$ by \cref{logisticRegressionIncohBound}. The second term is $\leq 2e^{-25}$ by \cref{logisticRegressionBGProp}. We now just need to analyze the first term:
  \begin{align*}
    \Pr\left[\max_{n \in [N]}\n{\frac{1}{N} \nabla_\theta f(x_n^T\theta^*, y_n)}_\infty \geq \frac{\lambda (\alpha - \MJLogr)}{4} - B_G \right].
  \end{align*}
  Plugging in the $\lambda$ given in \cref{logisticRegressionTheorem-app}, and using $\| \nabla_\theta f(x_n^T\theta^*, y_n)\|_\infty \leq \| x_n \|_\infty$, we can further upper bound this probability:
  $$ \leq \Pr\left[ \max_{n\in [N]} \n{x_n}_\infty \geq \left( \sqrt{2Cc_x^2 \log N} + \sqrt{50Cc_x^2} \right)\right].$$
  By part 1 of \cref{subEsubGMax}, this probability is $\leq 2e^{-25}$.
\end{proof}

\subsection{Logistic regression: $\lambda$ small enough}

In the case of linear regression, the LSSC held with $K=0$, so there was no work to be done in checking the bound in \cref{assum:lambdaSmall}; this is not the case for logistic regression. \citet{li:2015:sparsistency} prove that the LSSC holds here:

\begin{prop}[\citep{li:2015:sparsistency}] \label{logisticRegressionLSSC}
  The logistic regression model given above satisfies the $(\theta^*, N_{\theta^*})$ LSSC for any $\theta^*$ and $N_{\theta^*} = \R^D$ with a data-dependent constant $K = 1/4 (\max_n \n{x_n}_\infty)(\max_n \n{x_{nS}}_2^2)$. 
\end{prop}
\begin{proof}
  This is proved in Section 6.2 of \citet{li:2015:sparsistency}.
\end{proof}

We first show that this random $K$ is not too large with high probability under our random design:
\begin{prop} \label{logisticRegressionKBound}
  For $x_n \in \R^D$ comprised of i.i.d.\ $c_x$-sub-Gaussian random variables, the random variable $K = 1/4 (\max_n \n{x_n}_\infty)(\max_n \n{x_{nS}}_2^2)$ satisfies:
  \begin{equation}\Pr\left[ K \geq \frac{1}{4} \left(\sqrt{2c_x^2  \log(ND)} + \sqrt{50c_x^2} \right)\left( \deff + c_x^2 \deff (\log N + 26)\right) \right] \leq 3e^{-25}  \end{equation}
\end{prop}
\begin{proof}
  First, \cref{subEsubGMax} implies that $\max_n\n{x_n}_\infty \geq \sqrt{2c_x^2  \log(ND)} + \sqrt{50c_x^2 }$ with probability at most $2e^{-25}$, so the probability we are interested in is bounded by:
  \begin{equation} \leq \Pr\left[ \max_n \n{x_{nS}}_2^2 \geq \deff + c_x^2 \deff(\log N + 26) \right] + 2e^{-25}.  \label{lrTmp} \end{equation}
  Noting that $\n{x_{nS}}_2^2$ is the sum of $\deff$ $c_x^2$-sub-Exponential random variables, $\n{x_{nS}}_2^2$ is a $\deff c_x^2$-sub-Exponential random variable. \cref{subEsubGMax} then gives us that \cref{lrTmp} is bounded above by $3e^{-25}$.
\end{proof}

We can now prove the result we need, which is that $\lambda$ satisfies the upper bound in \cref{assum:lambdaSmall} with high probability.

\begin{lm}
  Take \cref{assum:randomX,assum:incoherenceAssumption,assum:DDeffScaling,assum:logisticRegressionLambdaMin}. Then, for the logistic regression setup in \cref{assum:logisticRegressionY} and $\lambda$ as given in \cref{logisticRegressionTheorem} and large enough $N$, we have:
  \begin{equation} \Pr\left[ \lambda \geq \frac{\min_n \lambda_{min}^2(\nabla^2_\theta \objn(\theta^*)_{SS})}{4 \left( \big(1- \max_n \max_{d \in S^c} \n{J_{nd}}_1\big) + 4\right)^2 }\frac{4\big(1 - \max_{d \in S^c} \n{J_{nd}}_1\big)}{K} \right] \leq 16e^{-25}  \end{equation}
\end{lm}
\begin{proof}
  Using \cref{logisticRegressionLambdaMin}, \cref{logisticRegressionIncohBound}, and \cref{logisticRegressionKBound}, the desired probability is $\leq 16e^{-25}$ if the following deterministic inequality holds:
  \begin{equation} \lambda \leq \frac{4(\alpha - \MJLogr)}{4(\alpha - \MJLogr + 4)^2} \frac{(L_{min} - Cc_x^2 \sqrt{\deff N})^2}{\left(\sqrt{2c_x^2 \log(ND)} + \sqrt{50c_x^2} \right)\left( \deff + c_x^2 \deff (\log N + 26)\right)} \label{lambdaMinCondition} \end{equation}
    We will lower bound the right hand side and show that $\lambda$ is less than this lower bound. Throughout, $C$ will be a generic constant that changes from line-to-line. First, as noted in the proof of \cref{logisticRegressionTheorem}, $\MJLogr = o(1)$ as $N\to\infty$, so that for large enough $N$, we have $(\alpha-\MJLogr)/(\alpha-\MJLogr + 4)^2 \geq (\alpha/2) / (\alpha/2 +4)^2$. Next, for large enough $N$, \cref{logisticRegressionLambdaMin} implies the denominator is greater than $CN$. Also for large enough $N$, the denominator is less than $C\deff \log N \sqrt{\log(ND)}$. We are left with checking the condition:
  \begin{equation}
    \lambda \leq C \frac{\alpha/2}{(\alpha/2 + 4)^2}  \frac{N^2}{\deff \log N \sqrt{\log(ND)}}. 
  \end{equation}
  Under \cref{assum:DDeffScaling}, we can upper bound the denominator to get a further lower bound on the right hand side:
  \begin{equation}
    \lambda \leq C \frac{\alpha/2}{(\alpha/2 + 4)^2} \frac{\log^{2/5}(N) N^2}{N^{2/5} \log N \sqrt{\log(N) + N}}. 
  \end{equation}
  Now, the right hand side goes to infinity as $N$ gets large, while the $\lambda$ given in \cref{logisticRegressionTheorem} goes to 0 as $N$ gets large. Thus, for sufficiently large $N$, \cref{lambdaMinCondition} holds.
\end{proof}

\end{document}